\documentclass{article}

\usepackage{arxiv}
\usepackage[utf8]{inputenc} 
\usepackage[T1]{fontenc}    
\usepackage{hyperref}       
\usepackage{url}            
\usepackage{booktabs}       
\usepackage{amsfonts}       
\usepackage{nicefrac}       
\usepackage{microtype}      
\usepackage{lipsum}		
\usepackage{graphicx}
\usepackage{doi}

\usepackage{amsmath, amsthm, amssymb}
\usepackage{graphicx}
\usepackage{lineno,hyperref}
\usepackage{graphicx}
\usepackage{url}
\usepackage{esvect}
\usepackage{mdframed}
\usepackage{lipsum}
\usepackage{lscape}
\usepackage{tensor}
\usepackage{soul}
\soulregister\cite7
\soulregister\ref7
\soulregister\pageref7

\usepackage{algorithm}
\usepackage{algorithmic}

\newcommand{\ignore}[1]{}

\newtheorem{definition}{Definition}
\newtheorem{theorem}{Theorem}
\newtheorem{example}{Example}
\newtheorem{lemma}{Lemma}

\title{Constrained Optimization with Qualitative Preferences
}

\author{
	Sultan Ahmed\\
		Department of Computer Science\\
	  University of Regina\\
	\texttt{ Sultan.Ahmed@uregina.ca}
	\AND
	\href{https://orcid.org/0000-0001-7381-1064}{\includegraphics[scale=0.06]{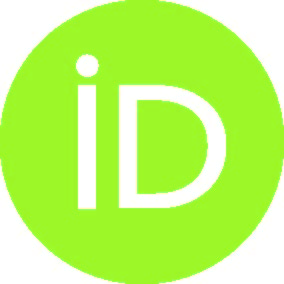}\hspace{1mm}Malek Mouhoub} \\
 	Department of Computer Science\\
	  University of Regina\\
	\texttt{mouhoubm@uregina.ca} \\
}

\renewcommand{\shorttitle}{Constrained Optimization with Qualitative Preferences  }

\hypersetup{
pdftitle={ Constrained Optimization with Qualitative Preferences},
pdfsubject={ },
pdfauthor={ Sultan Ahmed and Malek Mouhoub },
pdfkeywords={ Preference-based reasoning, Qualitative preferences, Constraint solving, Constrained optimization, CP-nets, LP-trees, Divide and Conquer  },
}

\begin{document}
\maketitle

\begin{abstract}
 The \emph{Conditional Preference Network} (CP-net) graphically represents user's qualitative and conditional preference statements under the \emph{ceteris paribus} interpretation. The constrained CP-net is an extension of the CP-net, to a set of constraints. The existing algorithms for solving the constrained CP-net require the expensive dominance testing operation. We propose three approaches to tackle this challenge. In our first solution, we alter the constrained CP-net by eliciting additional relative importance statements between variables, in order to have a total order over the outcomes. We call this new model, the constrained Relative Importance Network (constrained CPR-net). Consequently, We show that the Constrained CPR-net has one single optimal outcome (assuming the constrained CPR-net is consistent) that we can obtain without dominance testing. In our second solution, we extend the Lexicographic Preference Tree (LP-tree) to a set of  constraints. Then, we propose a recursive backtrack search algorithm, that we call Search-LP, to find the most preferable outcome. We prove that the first feasible outcome returned by Search-LP (without dominance testing) is also preferable to any other feasible outcome. Finally,  in our third solution, we preserve the semantics of the CP-net and propose a divide and conquer algorithm that compares outcomes according to dominance testing.
 \end{abstract}

\section{Introduction}

Many real world applications, including recommender systems \cite{Zanker2010,Jin2002PreferenceBasedCollobarativeFiltering}, product configurations~\cite{Brafman2002ProdConTCPnet}, and online auction systems~\cite{Sadaoui2016_pref_constraints}, require the management of both hard constraints and preferences. A preference represents user's desire~\cite{Zajonc1980feeling}, while a hard constraint, or simply a constraint, specifies legal combinations of assignments of values to some of the variables~\cite{PooleMackworth2017}. In such applications, helping users by providing the most preferable feasible outcome is of great interest as stated in recent Artificial Intelligence research~\cite{Boutilier2004constrainedCP_net,Alanazi2016constrained_cp_net,Brafman2006TCP_net_JAIR,Domshlak2003reasoningabout}.  Quantitative representation of preferences is well-known in Multi-Attribute Utility Theory~\cite{Keeney1993decision}. In this direction, a number of models extending the \emph{Constraint Satisfaction Problem} (CSP)  formalism~\cite{Dechter2003constraint} and including the Semiring-based CSP (SCSP)~\cite{Bistarelli1997SCSP} and the Valued CSP (VCSP)~\cite{Schiex1995VCSP,Bistarelli1999SCSP_VCSP} have been proposed to represent both constraints and quantitative preferences. However, considering qualitative preferences is interesting as these preferences are natural, and easier to elicit from users. Nevertheless, in a multi-attribute case, the number of outcomes is exponential in the number of variables. Thus, a direct assessment of the preference order is usually not practical. In this regard, many graphical models and logical languages have been proposed to represent qualitative preferences compactly~\cite{Baier2009AIMagazine,Kaci2011,Amor2016,Pigozzi2016}.

In particular, the \emph{Conditional Preference Network} (CP-net)~\cite{Boutilier2004_CPnetJAIR} is a well-known model to represent user's conditional preference statements under the \emph{ceteris paribus} (all else being equal) assumption. In general, an acyclic CP-net induces a strict partial order over the outcomes. This means that the preference order is irreflexive, asymmetric and transitive, however not necessarily complete, i.e., a pair  of outcomes is not always preferentially comparable. A noticeable extension of the CP-net is the Tradeoffs-enhanced CP-net (TCP-net)~\cite{Brafman2006TCP_net_JAIR}, in which an unconditional or conditional relative importance relation between a pair of variables is considered.  The TCP-net can capture more preferential information than the CP-net, yet pairs of outcomes are not necessarily comparable. Reasoning with CP-nets consists of two fundamental queries:  outcome comparison and outcome optimization. Outcome comparison queries are of two types: ordering query and dominance testing. The ordering query answers if an outcome is not preferable to another outcome. In this case, we say that the latter outcome is consistently orderable over the former outcome. In the acyclic CP-net, at least one of the two ordering queries between two outcomes can be answered in linear time in the number of variables.

On the other hand, dominance testing answers if an outcome is preferable to another outcome. In general, dominance testing in CP-nets is PSPACE-complete~\cite{Goldsmith2008computational} while it can be in NP (or even in P) under various assumptions~\cite{Boutilier2004_CPnetJAIR}. Boutilier et al.~\cite{Boutilier2004_CPnetJAIR} considered the problem of dominance testing as the problem of searching for an improving flipping sequence from one outcome to another outcome. An improving flip was defined using the CP-net semantics. An outcome can be improved to another, if they differ only on the values of a single variable and the variable gives a better value to the latter outcome than the former one. If there is an improving flipping sequence from one outcome to another, the latter outcome is preferred to the former one. Then, given the problem is intractable in acyclic CP-nets, some pruning techniques to reduce the search space such as the suffix fixing and the forward pruning were described. The suffix fixing rule guarantees that any improving flip induced by a variable such that this variable and its descendant set in the CP-net give the same value to both outcomes, can be pruned. In forward pruning, the values for each variable which are not relevant to a particular dominance query can be eliminated.  This requires a forward sweep through the network.  In the forward sweep, for each variable, the relevant values are selected and the rest of the network is restricted to these relevant values.  The above techniques can be applied to any generic search method to prune the search space without impacting the correctness or the completeness of the search procedure. On the other hand, Santhanam et al.~\cite{Santhanam2010dominanceModelChecking} translated dominance testing to model checking, and showed that existing model checking methods can be applied.  This translation facilitates the implementation of dominance testing. However, this does not establish additional computational advantages.

Outcome optimization is to obtain the most preferable outcome(s) which are also called the optimal outcome(s). For the acyclic CP-net or the acyclic TCP-nets, there is a single optimal outcome which can be found by the \emph{forward sweep} procedure~\cite{Boutilier2004_CPnetJAIR} that takes linear time in the number of variables. However, if the CP-net or the TCP-net is involved with a set of hard constraints, the most preferable outcome might be infeasible. In this case, the outcome optimization task is not trivial given that both the CP-net and the TCP-net represent a strict partial order over the outcomes. In constrained optimization using these models, the solving algorithms~\cite{Boutilier2004constrainedCP_net,Alanazi2016constrained_cp_net,Brafman2006TCP_net_JAIR,Ahmed2018PartialCPnet} suffer from a common drawback, i.e., requiring dominance testing between outcomes.  This is due to the fact that, when a new feasible outcome is found during the search process, this feasible outcome is considered as a Pareto optimal outcome if it is not preferentially dominated by any outcome of the existing Pareto set. In this paper, to the problem of requiring dominance testing in qualitative preference-based constrained optimization, we propose three solutions in which dominance testing is not needed or dominance testing is performed efficiently.

In our first solution, we propose a variant of the CP-net model, that will prevent us from using dominance testing, when constraints are considered. In this regard, we formally illustrate the relative importance relation between variables, which is induced by the parent-child relation of an acyclic CP-net. We show that the CP-net represents a total order of the outcomes, in particular cases, if and only if the induced variable importance order is total.  This motivates us to alter the CP-net model such that it guarantees to represent a total order over the outcomes.  In this regard, after constructing an acyclic CP-net, we determine every pair of variables in which the CP-net does not induce a relative importance relation, and we ask the user to explicitly provide a variable importance order for that pair.  We call the extended model the CP-net with Relative Importance (CPR-net).  We demonstrate that an acyclic CPR-net always represents a total order over the outcomes.  As a result, there is a single optimal outcome for the constrained optimization problem in which the preferences over the outcomes are described using the acyclic CPR-net.  Finally, we give an efficient algorithm that we call Search-CPR, to obtain this optimal outcome by utilizing the topological order of the acyclic CPR-net.  The formal properties of the algorithm are presented and discussed.  The main advantage of Search-CPR is that it does not require dominance testing between outcomes.

Note that related works on tackling preference-based constrained optimization without the need for dominance testing has been reported in the literature. In this regard, Wallace and Wilson~\cite{Wallace2006} described a method for representing constrained optimization problems with conditional preferences based on a lexicographic order of both variables and values. This representation requires to elicit a total order of the variables according to their importance and, in the method, searching for the most preferable feasible outcome by following this order does not require dominance testing. However, in some cases, this variable order can be conflicting to the variable order induced by the parent-child relationship of the conditional preferences, which indicates that some of the outcomes are incomparable or equally good.  In this case, the user is required to explicitly provide preferences over the incomparable outcomes, which introduces additional difficulty in the preference elicitation process. Instead of asking the user to provide a total order of variable importance, our method first determines the implicit variable importance orders induced by the conditional \emph{ceteris paribus} preference statements encoded in the CP-net. Then, the user is asked to explicitly provide the relative importance order between every pair of variables, which has not already been ordered by the CP-net.  Thus, in this case, the user provided variable order does not conflict with the CP-net induced order.


Freuder et al.~\cite{Freuder2003OrdinalCSP} studied the Ordinal CSP formalism which extends hard constraints to preferences. In that formalism, preferences are represented as a lexicographic order over variables and domain values. The authors proposed a backtrack search algorithm to obtain the most preferable feasible outcome. They considered both lexical order (induced by the lexicographic preferences) and ordinary CSP heuristics, to determine the variable order for instantiation, and provided appropriate trade-off between these two in terms of efficiency. One limitation of that formalism is that it does not consider conditional preferences.  \ignore{Consequently, the above formalism was extended~\cite{Wallace2006,Freuder2010Lexicographic} to incorporate conditional preferences. In the extended model, the user provides: (1) a unique total order over the variables based on the variable importance, and (2) conditional preferences over the variable domain values.  Given that the user provided variable order corresponds to the one induced by the parent-child relationship of the conditional preferences, the authors showed that the algorithms originally devised for the Ordinal CSP can be applied in this extended model as well. The model is similar to the CP-net provided that the lexicographic variable order guarantees a total order over the outcomes in the former one.  Yet, this does not consider the fact that the variable order can also be conditioned for some value of other variables. } Given this limitation, the Lexicographic Preference tree (LP-tree)~\cite{Booth2010LPtree} is a more general formalism as it considers the fact that both lexicographic variable order and value order can be conditioned on the actual value of some more important variables.  Instead of representing a unique total order over the variables, the LP-tree represents a set of hierarchical orders over the variables which are also total order. This has inspired us to reply on LP-trees for our second solution. More precisely, we extend the LP-tree graphical model to a set of hard constraints. We call the new model the Constrained LP-tree. To our best knowledge, this is the first time such model is proposed. In the model, we define a most preferable feasible outcome as an outcome, which is feasible and no other feasible outcome is preferable to the outcome. In order to find the most preferable feasible outcome, we propose a recursive backtrack search algorithm that we call Search-LP.  Search-LP begins with the instantiation of the most preferred value of the root node of the LP-tree.  After strengthening the constraints, the algorithm checks the consistency of the new set of constraints. If this set is inconsistent, the branch for this assigned value is terminated, and Search-LP continues with the next values of the root node, according to the preference order, until a consistent set of constraints is found.  Then, the partial assignment induced by the root node value and the new set of constrains is obtained, and the termination criterion is checked. If all variables are instantiated, this feasible outcome that we prove to be the optimal one, is returned. If the termination criterion is not met, we reduce the LP-tree by removing the instantiated variables.  We prove that the Reduced LP-tree is compatible with the original LP-tree, i.e., the preferences induced by the Reduced LP-tree are also held by the original LP-tree, given the instantiation of the removed variables. Then, Search-LP is called recursively, until the termination criterion is met. To every recursive call, the instantiation of the removed variables is forwarded.  Finally, if no feasible outcome exists, i.e., the CSP is inconsistent, Search-LP stops by returning $null$.

Search-LP produces the most preferable feasible outcome, while the underlying CSP might have an exponential number of feasible outcomes\ignore{ and all such outcomes are needed to produce by considering the CSP standalone}. In this sense, we can say that solving the Constrained LP-tree is no harder than solving the underlying CSP, given that dominance testing is not needed as it is the case for the Constrained CP-net and the Constrained TCP-net where this operation is of exponential cost~\cite{Boutilier2004_CPnetJAIR,Goldsmith2008computational,Brafman2006TCP_net_JAIR}. Saying this, we cannot apply variable ordering heuristics~\cite{Mouhoub2011VariableValueOrderingCSP,Yong2017VariableValueOrderCSP} as it was done in the case of Constrained CP-nets~\cite{Alanazi2016constrained_cp_net}, to improve the performance of the search in practice. This is due to the fact that Search-LP instantiates the variables based on a hierarchical order of the variables defined by the LP-tree.

In our third solution, we consider the problem of dominance testing for acyclic CP-nets, and we propose a divide and conquer algorithm that we call Acyclic-CP-DT, to answer any dominance query. We observe that some upper portion of an arbitrary topological order which gives the same value to both outcomes, are not significantly needed to answer the dominance query. Given the first variable in the topological order which gives two different values to the outcomes, we divide the problem into two sub problems -- one for each value.  For each sub problem, we build a sub CP-net by removing the variable and its ancestors from the original CP-net.  Then, Acyclic-CP-DT is called recursively for the sub CP-nets until it reaches to a base condition.  By evaluating the return value of the sub calls, Acyclic-CP-DT determines and returns an answer to the original query.  We formally show that this evaluation method is correct, and a query can be answered efficiently in particular cases. However, the completeness portion of the algorithm is very complicated.  In the worst case, a significant portion of the search space, related to the distinct values given by some bottom portion of a topological order, needs to be searched.  That is why all instances of the problem class are still not tractable.  Nevertheless, we show that Acyclic-CP-DT is computationally an improvement to the existing methods of the dominance testing~\cite{Boutilier2004_CPnetJAIR}.

The rest of the paper is organized as follows. In Section~\ref{SEC_BK}, we provide the necessary preliminary knowledge. In Section~\ref{SEC_CON_CPR_NET}, we present both the CPR-net and the Constrained CPR-net. The Constrained LP-tree is then described in Section~\ref{SEC_CON_LP_TREE}. In Section~\ref{Acyclic_CP_DT_Section_Algorithm_Example}, we describe the divide and conquer algorithm to perform dominance testing. We list concluding remarks and some future research directions in Section~\ref{SEC_CON_FUTURE_WORK}. Note that this paper is a comprehensive study on qualitative preference-based constrained optimization, and extends the previous works we conducted in this regard~\cite{Ahmed2018CPRnet,Ahmed2019LPtreeSMC,Ahmed2019DivedeConquerDT}. The main purpose here is to provide the reader with the alternatives to consider when tackling problems of this kind.

\section{Background Knowledge}
\label{SEC_BK}
\subsection{Preference and relative importance relations}

We assume a set of variables $V=\{X_1,X_2,\cdots,X_n\}$ with the finite domains $D(X_1),D(X_2),\cdots ,D(X_n)$. We use $D(\cdot)$ to denote the domain of a set of variables as well.  The decision maker wants to express preferences over the complete assignments on $V$.  Each complete assignment can be seen as an outcome of the decision maker's action.  The set of all outcomes is denoted by $O$.  A preference order $\succ$ is a binary relation over $O$.  For $o_1,o_2\in O$, $o_1\succ o_2$ indicates that $o_1$ is strictly preferred to $o_2$.  The preference order $\succ$ is necessarily a partial order, i.e., $\succ$ is irreflexive, asymmetric and transitive.  The preference order $\succ$ is a total order, if $\succ$ is also complete, i.e., for every $o_1,o_2\in O$, we have either $o_1\succ o_2$, $o_2\succ o_1$, or $o_1$ and $o_2$ are preferentially incomparable.

The size of $O$ is exponential in the number of variables.  Therefore, direct assessment of the preference order is usually impractical.  In this case, the notions of \emph{preferential independence} and \emph{conditional preferential independence} play a key role to represent the preference order compactly, at least if the preference order is partial.  These are standard and well-known notions of independence in multi-attribute utility theory~\cite{Keeney1993decision,Boutilier2004_CPnetJAIR,Brafman2006TCP_net_JAIR}.

\begin{definition}~\cite{Boutilier2004_CPnetJAIR}
Let $x_1,x_2\in D(X)$ for some $X\subseteq V$, and $y_1,y_2\in D(Y)$, where $Y=V-X$.  We say that $X$ is preferentially independent of $Y$ iff, for all $x_1,x_2,y_1,y_2$, we have that $x_1 y_1\succ x_2 y_1\Leftrightarrow x_1 y_2\succ x_2 y_2$.
\end{definition}

If this preferential independence holds, we say that $x_1$ is preferred to $x_2$ \emph{ceteris paribus} (all else being equal).  This implies that the decision maker's preferences for different values of $X$ do not change as other attributes vary.  The analogous definition of conditional preferential independence is as follows.

\begin{definition}~\cite{Boutilier2004_CPnetJAIR} Let $X$, $Y$ and $Z$ be a partition of $V$ and let $z\in D(Z)$. We say that $X$ is conditionally preferentially independent of $Y$ given $z$ iff, for all $x_1,x_2\in D(X)$ and $y_1,y_2\in D(Y)$, $x_1 y_1z\succ x_2y_1z\Leftrightarrow x_1y_2z\succ x_2 y_2 z$.  $X$ is conditionally preferentially independent of $Y$ given $Z$, iff $X$ is conditionally preferentially independent of $Y$ given every assignment $z\in D(Z)$.
\end{definition}

We now define the notion of \emph{relative importance} of variables.  The ordering of the outcomes induced by this notion is relatively stronger than that of the preferential independence~\cite{Brafman2006TCP_net_JAIR}. For example, a reader wants to borrow a book from the following available options in a library: $\{Fiction,Paper\}$, $\{Fiction,Electronic\}$, $\{Nonfiction,Paper\}$ and $\{Nonfiction,Electronic\}$.  In this case, the attributes are:  $Genre=\{Fiction, Nonfiction\}$ and $Media=\{Paper, Electronic\}$.  The user's preferences on $Genre$ and $Media$ are not affected by each other, i.e., $Genre$ and $Media$ are preferentially independent.  The user expresses the following preferences:  $Fiction\succ Nonfiction$ and $Paper\succ Electronic$. Obviously, $\{Fiction,Paper\}$ is the most preferred outcome and $\{Nonfiction,Electronic\}$ is the least preferred outcome in this scenario. However, we do not know the order between $\{Fiction,Electronic\}$ and $\{Nonfiction,$ $Paper\}$.  This case is typical for independent variables.  Using the \emph{ceteris paribus} semantics, we can always compare two outcomes when they differ on a single variable. However, we cannot always compare them when they differ by more than one variable. The relative importance of variables can address some of such comparisons.  For example, if the reader specifies that having a better $Genre$ is more important than having a better $Media$, then we get:  $\{Fiction,Electronic\}\succ \{Nonfiction,Paper\}$.

\begin{definition}~\cite{Brafman2006TCP_net_JAIR} Let a pair of variables $X$ and $Y$ be mutually preferentially independent given $W=V-\{X,Y\}$.  We say that $X$ is more important than $Y$, denoted by $X\rhd Y$, iff for every assignment $w\in D(W)$ and for every $x_i,x_j\in D(X)$, $y_a,y_b\in D(Y)$, such that $x_i\succ x_j$ given $w$, we have that:  $x_iy_aw\succ x_jy_bw$.
\end{definition}

The relative importance of variables is closely related to the lexicographic preferences~\cite{Fishburn1974Lexicographic,Freuder2010Lexicographic}.  In lexicographic preferences, the relative importance order of variables is considered to be a total order, i.e., there exists a relative importance order between every pair of variables.

\subsection{CP-nets}

A \emph{Conditional Preference Network} (CP-net)~\cite{Boutilier2004_CPnetJAIR} graphically represents user's conditional preference statements using the notions of preferential independence and conditional preferential independence under the \emph{ceteris paribus} assumption.  A CP-net consists of a directed graph, in which, preferential dependencies over the set $V$ of variables are represented using directed arcs.  An arc $\vv{(X_i,X_j )}$ for $X_i,X_j\in V$ indicates that the preference orders over $D(X_j)$ depend on the actual value of $X_i$.  For each variable $X\in V$, there is a \emph{Conditional Preference Table} (CPT) that represents the preference orders over $D(X)$ for each $p\in D(Pa(X))$, where $Pa(X)$ is the set of $X$'s parents.  Note that, nothing prevents the CP-net to be cyclic, however an acyclic CP-net always represents a preference order (at least a partial order) over the outcomes while a cyclic CP-net does not guarantee it.

\begin{definition}~(CP-net)~\cite{Boutilier2004_CPnetJAIR} A Conditional Preference Network (CP-net) $N$ over variables $V=\{X_1,X_2,\cdots ,X_n \}$ is a directed graph over $X_1,X_2,\cdots ,X_n$ whose nodes are annotated with $CPT(X_i )$ for each $X_i\in V$.
\end{definition}

\begin{example} An acyclic CP-net with four variables, $A$, $B$, $C$ and $D$, is shown in Figure~\ref{label_cp_netfig}, where $D(A)=\{a_1,a_2\}$, $D(B)=\{b_1,b_2,b_3\}$, $D(C)=\{c_1,c_2\}$ and $D(D)=\{d_1,d_2\}$.  The preference order over $D(A)$  is $a_1\succ a_2$, and depends on no other variable.  The preferences over $D(B)$ depend on the actual value of $A$.  Given $A=a_1$, the preference order over $D(B)$ is $b_1\succ b_2\succ b_3$; while given $A=a_2$, the order is $b_3\succ b_2\succ b_1$.  The preferences over $D(C)$ depend on $B$.  Given $B=b_1$ or $B=b_3$, the user has the same preference $c_1\succ c_2$ on $D(C)$; while the preference is $c_2\succ c_1$ for $B=b_2$.  The variable $D$ is preferentially independent from the other variables.  The preference order on $D(D)$ is $d_1\succ d_2$. $\hfill \square$
\end{example}

\begin{figure}[htbp]
\centerline{\includegraphics[width=0.8\textwidth]{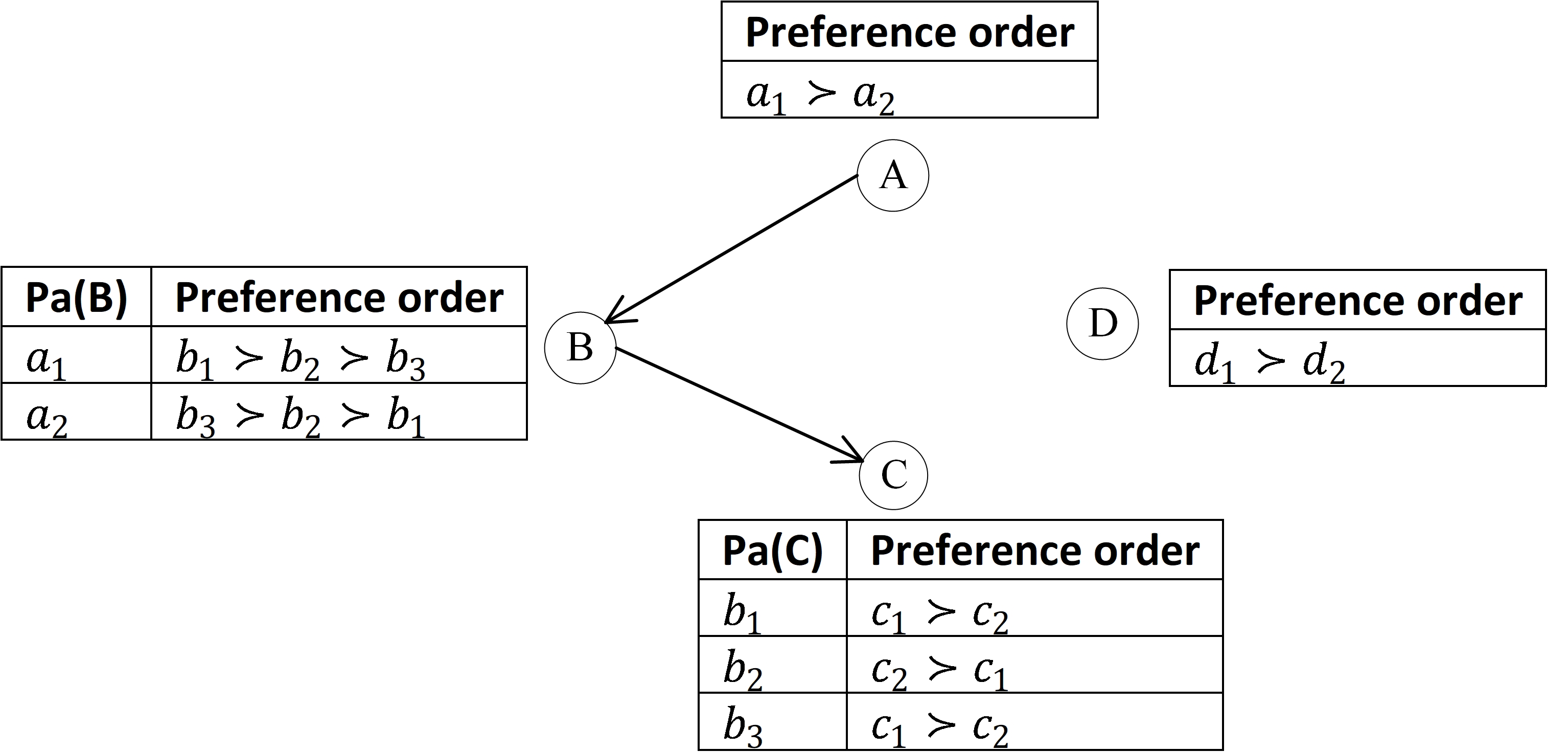}}
\caption{A CP-net.}
\label{label_cp_netfig}
\end{figure}

The semantics of a CP-net is defined in terms of the set of preference orders over the outcomes, which are consistent with the set of preferences imposed by the CPTs.  A preference order $\succ$ on the outcomes of a CP-net $N$ satisfies the CPT of a variable $X$, iff $\succ$ orders any two outcomes that differ only on the value of $X$ consistently with the preference order on $D(X)$ for each $p\in D(Pa(X))$.  $\succ$ satisfies $N$ iff $\succ$ satisfies each CPT of $N$.  If $o_1$ and $o_2$ are two outcomes of $N$, we say that $N$ entails $o_1\succ o_2$, written as $N\models o_1\succ o_2$, iff $o_1\succ o_2$ holds in every preference order that satisfies $N$. Similarly, $N\not \models o_1\succ o_2$ indicates that $o_1\succ o_2$ does not hold in every preference order that satisfies $N$. With respect to $N$, two outcomes $o_1$ and $o_2$ can stand in one of the three following options: $N\models o_1\succ o_2$; or $N\models o_2\succ o_1$; or $N\not\models o_1\succ o_2$ and $N\not \models o_2\succ o_1$. The third option, specifically, indicates that the network does not have enough information to compare $o_1$ and $o_2$. There are two types that two outcomes can be compared. First, determining if $N\models o_1\succ o_2$ holds or not, is called the dominance testing, which is generally a PSPACE-complete problem~\cite{Boutilier2004_CPnetJAIR,Goldsmith2008computational}.  On the other hand, an ordering query determines if an outcome $o_1$ is not preferable to another outcome $o_2$, i.e., $N\not\models o_1\succ o_2$. If $N\not\models o_1\succ o_2$ is true, we say that $o_2$ is consistently orderable over $o_1$, which we denote as $N\vdash_{oq} o_2\gg o_1$. This means that there exists at least a consistent preference order of $N$ in which $o_2\succ o_1$ holds. For an acyclic CP-net, at least one of the ordering queries ($o_1\gg o_2$ or $o_2\gg o_1$) between the outcomes $o_1$ and $o_2$ can be answered in linear time in the number of variables.

If two outcomes $o_1$ and $o_2$ of an acyclic CP-net $N$ differ only on the values of a variable $X$, and $X$ gives a better value to $o_1$ than $o_2$, we call that there exists an improving flip from $o_2$ to $o_1$.  This improving flip also implies $N\models o_1\succ o_2$.  A sequence of improving flips from one outcome to another outcome implies that the latter outcome is preferred to the former outcome. An improving flipping sequence from $o_2$ to $o_1$ is irreducible, if deleting any entry except $o_1$ and $o_2$ does not produce an improving flipping sequence. Let $F$ be the set of all irreducible improving sequences among outcomes.  We denote by $MaxFlip(X)$ the maximal number of times that a variable $X$ can change its value in any improving flipping sequence in $F$.

An acyclic CP-net can be satisfied by more than a preference order over the outcomes.  However, it is guaranteed to have a single optimal outcome of the CP-net which preferentially dominates every other outcome.  This optimal outcome can be obtained using the \emph{forward sweep} procedure~\cite{Boutilier2004_CPnetJAIR}, that takes linear time in the number of variables.  For example, using the forward sweep procedure, we can easily find the optimal outcome $a_1b_1c_1d_1$ for the CP-net of Figure~\ref{label_cp_netfig}.

\subsection{LP-trees}

In lexicographic preference order, we consider both variable and value orders. The variable order is based on the relative importance of the variables. A variable $X$ is more important than another variable $Y$ iff having a better value for $X$ is preferred to having a better value for $Y$.  The decision maker specifies a total order on the variables first, and then a total order on the values of each variable.  Here, both variable and value orders can depend on the actual value of more important variables.  To represent such conditional lexicographic preferences, the Lexicographic Preference Tree (LP-tree) has been proposed~\cite{Booth2010LPtree}.

\begin{definition}~\cite{Booth2010LPtree} A Lexicographic Preference tree (LP-tree) $L$ over the variable set $V$ is a tree such that the following statements are true.

\begin{enumerate}
  \item Every node is labelled with an attribute $X\in V$.  The function $An(X)$ indicates the set of ancestor nodes of node $X$, while $De(X)$ indicates the set of descendant nodes of node $X$.
  \item Every arc $\vv{(X,Y)}$ represents a relative importance relation, i.e., the parent node $X$ is more important than the child node $Y$ given the instantiation of $An(X)$.
  \item Every arc $\vv{(X,Y)}$ is labeled with at least a value of $X$, indicating that the preferences represented by the subtree of the child node $Y$, i.e., the subtree where $Y$ is the root node, hold given these values of $X$ and the instantiation of $An(X)$.  The subtree of $Y$ is unique to the subtrees of other children of $X$.
  \item Every node $X$ is labelled with a \emph{Conditional Preference Table} (CPT).  The $CPT(X)$ represents a preference order on $D(X)$ for every instantiation of $An(X)$.
\end{enumerate}
\end{definition}

\begin{example} Let a customer needing to choose his dinner has the following configuration:  he needs to choose meat ($a_1$) or fish ($a_2$) for the main course ($A$), vegetable soup ($b_1$) or fish soup ($b_2$) as soup ($B$), and red wine ($c_1$) or white wine ($c_2$) as drink ($C$).  The customer's preferences are as follows.

Meat ($a_1$) is preferred to fish ($a_2$) as the main course regardless on the preferences of soup and drink.  If he chooses meat ($a_1$) as the main course, then having a better soup is more important than having a better drink.  He prefers vegetable soup ($b_1$) to a fish soup ($b_2$), while preferences on drink are conditioned on the choice of soup.  If vegetable soup ($b_1$) is chosen, then he prefers red wine ($c_1$) to white wine ($c_2$); else he prefers white wine ($c_2$) to red wine ($c_1$). On the other hand, if he chooses fish ($a_2$) as the main course, then having a better drink is more important than having a better soup.  He prefers a red wine ($c_1$) to a white wine ($c_2$), and also he prefers a vegetable soup ($b_1$) to a fish soup ($b_2$). In this case, note that his preferences on soup and drink are mutually independent.

The above preferences are represented using the LP-tree in Figure~\ref{fig_my_dinner_lp_tree}. Each node represents a variable while each arc corresponds to a relative importance relation. Arc $\vv{(A,B)}$ indicates that $A$ is more important than $B$. Given $A=a_1$, arc $\vv{(B,C)}$ indicates that $B$ is more important than $C$; while for $A=a_2$, we have that $C$ is more important than $B$.  This representation explicitly indicates a relative importance relation. On the other hand, the CPT of a variable represents the local preferences on the variable.  For the branch $A=a_1$,  $CPT(C)$ represents the preferences on $D(C)$ which are conditioned on the actual value of $B$. $\hfill \square$
\label{example_my_dinner_lp_tree}
\end{example}

\begin{figure}
\centering
\includegraphics[width=0.4\textwidth]{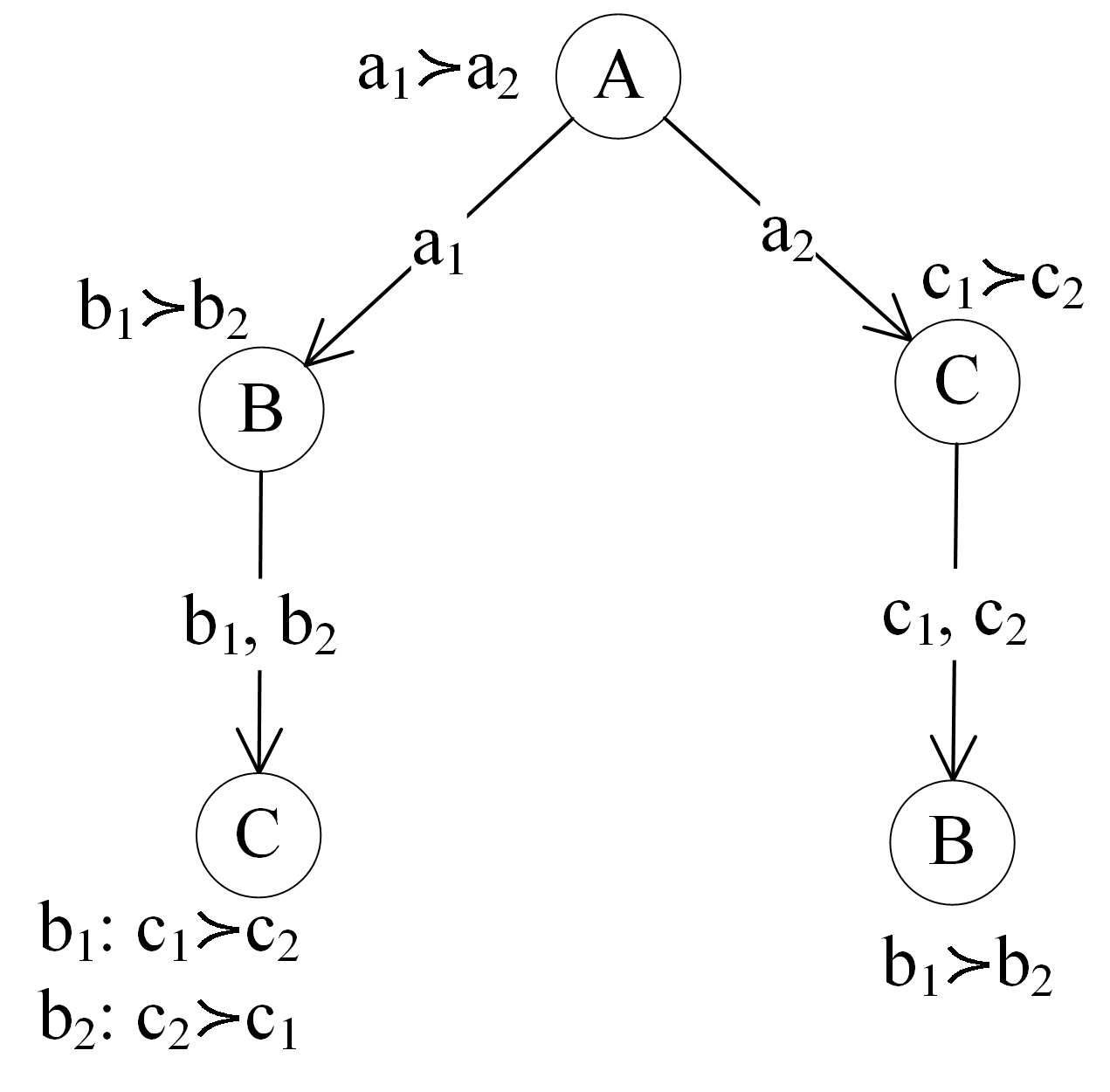}
\caption{LP-tree corresponding to Example~\ref{example_my_dinner_lp_tree}.}
\label{fig_my_dinner_lp_tree}
\end{figure}

The preference order between two outcomes with respect to an LP-tree is defined below.

\begin{definition}~\cite{Booth2010LPtree} Let $o_1$ and $o_2$ be two outcomes of an LP-tree $L$. Let $X$ be the node such that $An(X)$ gives the same value $R$ to both $o_1$ and $o_2$. Node $X$ gives two different values $x_1$ and $x_2$ to $o_1$ and $o_2$ correspondingly, and $X$ exists in the subtree corresponding to $An(X)=R$. Given the instantiation of $An(X)$, we get $o_1\succ o_2$ (denoted as $L\models o_1\succ o_2$) iff $CPT(X)$ represents $x_1\succ x_2$. Similarly, we get $L\models o_2\succ o_1$ iff $CPT(X)$ represents $x_2\succ x_1$.
\label{definition_lex_order}
\end{definition}

\begin{example}
In Figure~\ref{fig_my_dinner_lp_tree}, consider two outcomes $a_1 b_1 c_1$ and $a_1 b_2 c_2$.  In this case, $B$ is the node that gives two different values to the outcomes given that the ancestor variable $A$ gives the same value $a_1$ to both outcomes.  For $A=a_1$, $CPT(B)$ represents $b_1\succ b_2$.  Therefore, by Definition~\ref{definition_lex_order}, we get $a_1 b_1 c_1\succ a_1 b_2 c_2$. $\hfill \square$
\end{example}

Lemma~\ref{lemma_lex_total_order} is a direct consequence of Definition~\ref{definition_lex_order}, i.e., any two outcomes can be ordered by Definition~\ref{definition_lex_order}.

\begin{lemma}~\cite{Booth2010LPtree} Every LP-tree $L$ represents a total order over the outcomes.
\label{lemma_lex_total_order}
\end{lemma}

\begin{example} Figure~\ref{fig_my_dinner_lp_tree} represents the total order $a_1 b_1 c_1\succ a_1 b_1 c_2\succ a_1 b_2 c_2\succ a_1 b_2 c_1\succ a_2 c_1 b_1\succ a_2 c_1 b_2\succ a_2 c_2 b_1\succ a_2 c_2 b_2$. $\hfill \square$
\label{example_my_dinner_lp_tree_total_order}
\end{example}

The optimal outcome for an LP-tree is the one which is preferred to other outcomes.  In the example above, $a_1 b_1 c_1$ is the optimal outcome.

\subsection{Constraint satisfaction}

A hard constraint, or simply a constraint, specifies legal combinations of assignments of values to some of the variables. A Constraint Satisfaction Problem (CSP) consists of a set of variables $V$, their domains $D$, and a set of constraints $C$. A complete assignment on $V$ satisfies, or is consistent with, a constraint $c\in C$ on $U\subseteq V$ if the projection\footnote{Let $d_1d_2\cdots d_n$ be a complete assignment on variables $X_1X_2\cdots X_n$. The projection of $d_1d_2\cdots d_n$ on some subset of variables $X_{i_1}X_{i_2}\cdots X_{i_l}$ is $d_{i_1}d_{i_2}\cdots d_{i_l}$.} of the complete assignment on $U$ belongs to $c$. The complete assignment is consistent with the CSP if it is consistent with all constraints in $C$. In this case, we call the complete assignment a feasible outcome. A CSP is consistent if it has at least one feasible outcome. The tasks with respect to a CSP consists of the following questions: is the CSP consistent, finding a feasible outcome, or finding all feasible outcomes. On the other hand, a Constrained Optimization Problem (COP) is a CSP together with an objective function, which is subject to be minimized or maximized, such as we minimize a cost function and maximize a profit function. The tasks in COPs are to find one or all optimal outcome(s), where an optimal outcome is a feasible outcome that optimizes the objective function.

\begin{example} Consider three variables $A$, $B$, and $C$ with domains $D(A)=\{a_1,a_2\}$, $D(B)=\{b_1,b_2\}$, and $D(C)=\{c_1,c_2\}$. We also consider the constraints $\{A=a_1\leftrightarrow B=b_1\}$ and $\{A=a_2\leftrightarrow B=b_2\}$ on $AB$, and the constraints $\{A=a_1\leftrightarrow C=c_1\}$ and $\{A=a_2\leftrightarrow C=c_2\}$ on $AC$. The complete assignment $a_1b_1c_1$ satisfies the constraints on $AB$ as projection of $a_1b_1c_1$ on $AB$ is $a_1b_1$, which is legal with respect to $\{A=a_1\leftrightarrow B=b_1\}$. Similarly, $a_1b_1c_1$ satisfies the constraints on $AC$, and we say that $a_1b_1c_1$ satisfies the CSP. On the other hand, $a_1b_1c_2$ does not satisfy the constraints on $AC$ as projection of $a_1b_1c_2$ on $AC$ is $a_1c_2$, which is not legal with respect to the constraints on $AC$. $\hfill \square$
\end{example}

\section{Constrained CPR-nets}
\label{SEC_CON_CPR_NET}

In this section, first, we illustrate the implicit relative importance order of the variables induced by the dependency edges in CP-nets. Then, we provide a necessary and sufficient condition when a CP-net represents a total order over the outcomes. Second, we describe our CPR-net model that always represents a total order over the outcomes. Finally, we apply the model in constrained optimization.

\subsection{Dependency and ordering in CP-nets}
\label{section_dependency}

We identify two types of preferential dependencies (called partial dependency and total dependency) induced by a CP-net.  Then, we show that a total dependency between two variables also indicates a relative importance relation.  By using such relations, we provide a necessary and sufficient condition of a CP-net for its outcomes to be totally ordered.

\subsubsection{Partial dependency and total dependency}

An arc $\vv{(X,Y)}$ in a CP-net does not necessarily indicate that, for every $x\in D(X)$, there is a unique preference order over $D(Y)$, i.e., for two different values of $X$, $CPT(Y)$ can represent the same preference order. For example, let us consider arc $\vv{(B,C)}$ in Figure~\ref{label_cp_netfig}. $CPT(C)$ represents the same preference order $c_1\succ c_2$ for $B=b_1$ and $B=b_3$. In this regard, it is guaranteed by definition of preferential dependency that there are two partitions of $D(B)$, $\{b_1,b_3\}$ and $\{b_2\}$, such that $CPT(C)$ represents a unique preference order over $D(C)$, for each. When such a partition contains more than one value, the dependency does not hold for the values, i.e., the preference order is the same given the values. In this case, it is reasonable to argue that the preferences over $D(Y)$ partially depend on the values of $X$.

\begin{definition}~(Partial dependency arc)~A CP-net arc $\vv{(X,Y)}$ is a partial dependency arc, iff $CPT(Y)$ represents the same preference relation over any $y_1,y_2\in D(Y)$ for any two or more values of $X$, given any $p\in D(Pa(Y)-\{X\})$.
\end{definition}

\begin{example} Consider the CP-net of Figure~\ref{label_cp_netfig}. The arc $\vv{(B,C)}$ is a partial dependency arc, since for $B=b_1$ and $B=b_3$, we have the same preference order over $c_1$ and $c_2$.
\end{example}

\begin{definition}~(Totally dependent arc)~$\vv{(X,Y)}$ in a CP-net is a totally dependent arc, iff $CPT(Y)$ represents different preference relations over every $y_1,y_2\in D(Y)$ for every $x\in D(X)$ and $p\in D(Pa(Y)-\{X\})$.
\end{definition}

\begin{example} Consider arc $\vv{(A,B)}$ in Figure~\ref{label_cp_netfig}. $CPT(B)$ gives $b_1\succ b_2\succ b_3$ and $b_3\succ b_2\succ b_1$ for $A=a_1$ and $A=a_2$ correspondingly, which are unique. Therefore, arc $\vv{(A,B)}$ is totally dependent.
\end{example}

\begin{definition}~(Totally dependent CP-net)~A CP-net is called totally dependent iff every arc of the CP-net is totally dependent.
\end{definition}

\begin{figure}[t]
\centerline{\includegraphics[width=0.5\textwidth]{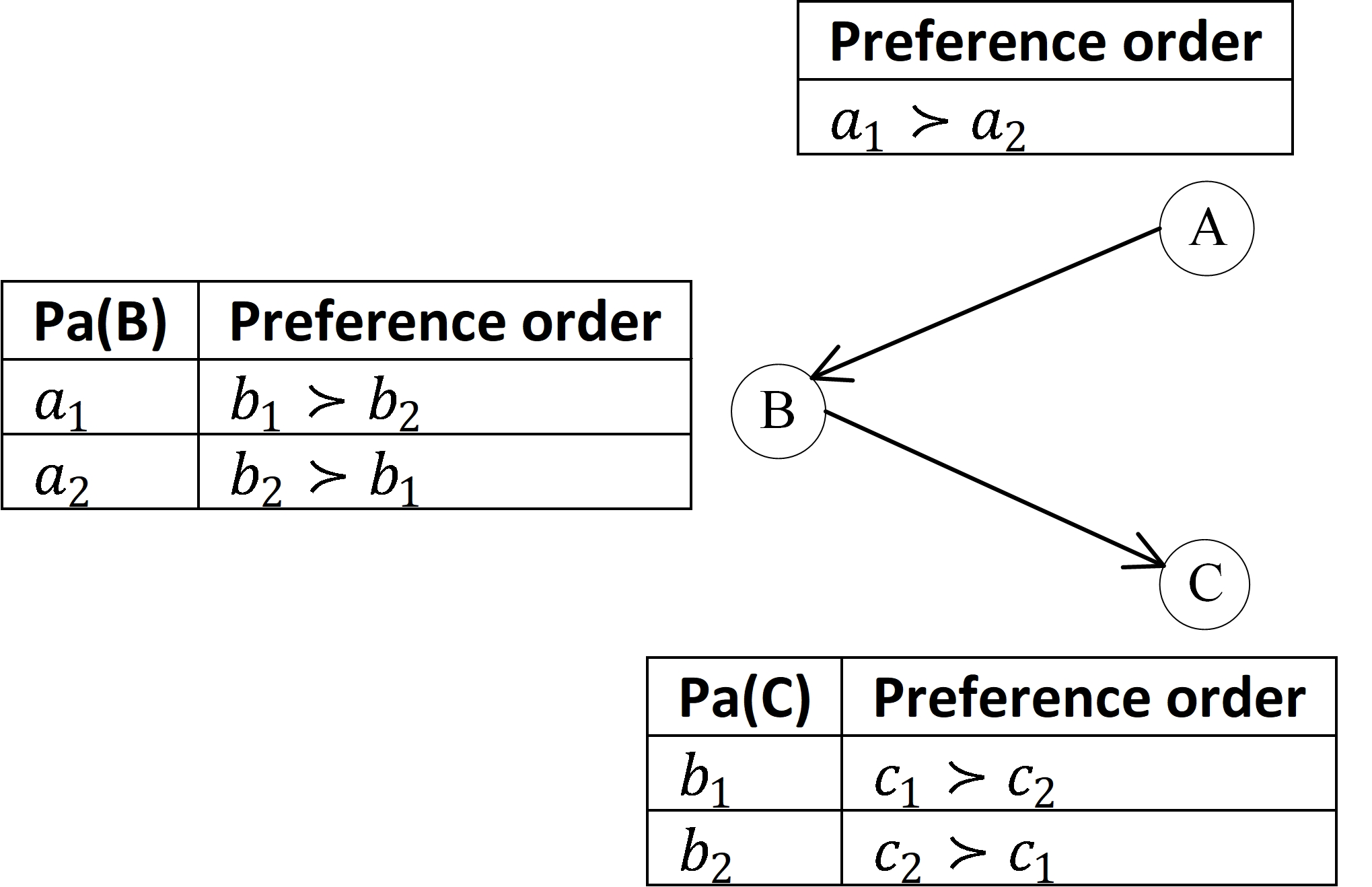}}
\caption{A totally dependent CP-net.}
\label{Label_totallydependentcpnet_fig}
\end{figure}

\begin{example} Consider the CP-net in Figure~\ref{Label_totallydependentcpnet_fig}, that is obtained by restricting the domain of $B$ from Figure~\ref{label_cp_netfig}, and deleting variable $D$. The CP-net is totally dependent given that every arc of this graph is totally dependent.
\end{example}

Note that the notion of CP-net is less restrictive than the notion of totally dependent CP-net. Indeed, due to the existence of partial dependency arcs in a general CP-net, some outcomes become incomparable. For example, in the CP-net of Figure~\ref{label_cp_netfig}, any outcome containing $a_1 b_1 c_2$ is incomparable to any other outcome containing $a_1 b_3 c_1$.  This is due to the existence of the partial dependency arc $\vv{(B,C)}$. On the other hand, solving Constrained CP-nets consists of finding the set of Pareto optimal outcomes, which requires dominance testing in addition to checking the feasibility of the outcomes. Generally, dominance testing is PSPACE-complete~\cite{Goldsmith2008computational}, and is needed between the considered feasible outcome and every element of the Pareto set, in the worst case. This is an expensive procedure knowing that the size of Pareto set can be exponential. It is obvious to see that we can overcome this challenge by reducing the Pareto set to one element and this can be achieved by imposing a total order on the outcomes when building the CP-net.  These obviously will give extra difficulty for the elicitation process, however it will also significantly improve the performance of the constrained optimization algorithms.  The rest of this section is based on totally dependent CP-nets, and we use CP-nets and totally dependent CP-nets interchangeably.

\subsubsection{Induced variable order}

In a CP-net, parent preferences have higher priority than children preferences~\cite{Boutilier2004_CPnetJAIR}. This induces a relative importance relation between the parent and the child, i.e., the parent is more important than the child. We formally explore this using the lemma below.

\begin{lemma} Let $N$ be a totally dependent acyclic CP-net on $V$.  $X\rhd Y$ is induced by $N$ if and only if $\vv{(X,Y)}$ exists in $N$.
\label{lemma_relative_importance_directed_arc}
\end{lemma}

\begin{proof}~($\Rightarrow$)  Let us assume that $\vv{(X,Y)}$ does not exist in $N$.  Therefore, $X$ and $Y$ are mutually preferentially independent of $X$ given $W=V-\{X,Y\}$ (where $W$ might also be empty).  For every $x_1,x_2\in D(X)$, every $y_1,y_2\in D(Y)$ and every $w\in D(W)$ such that $x_1\succ x_2$ and $y_1\succ y_2$, we have that $wx_1 y_2$ and $wx_2 y_1$ are not preferentially comparable.  Therefore, neither $X$ is more important than $Y$ nor $Y$ is more important than $X$.

($\Leftarrow$) Let us consider that $\vv{(X,Y)}$ exists in $N$.  Since $\vv{(X,Y)}$ is a totally dependent arc, for every $x_1,x_2\in D(X)$ and every $y_1,y_2\in D(Y)$ such that $x_1\succ x_2$, $CPT(Y)$ represents two complimentary preference relations over $y_1,y_2$ for $x_1$ and $x_2$ correspondingly.  Let, these be as follows:

\hangindent=.5cm \hangafter=1 $(x_1:y_1\succ y_2)\land (x_2:y_2\succ y_1)$

\noindent $\Rightarrow (x_1 y_1\succ x_1 y_2)\land (x_2 y_2\succ x_2 y_1)$

\noindent $\Rightarrow$ $(x_1 y_1\succ x_1 y_2)\land (x_1 y_2\succ x_2 y_2)\land (x_2 y_2\succ x_2 y_1)$ [Since $X$ is independent of $Y$ and we have $x_1\succ x_2$]

\noindent $\Rightarrow ((x_1 y_1\succ x_1 y_2)\land (x_1 y_2\succ x_2 y_2))\land ((x_1 y_2\succ x_2 y_2)\land (x_2 y_2\succ x_2 y_1))$

\noindent $\Rightarrow (x_1 y_1\succ x_2 y_2 )\land (x_1 y_2\succ x_2 y_1 )$  [By transitivity]

\noindent $\Rightarrow X\rhd Y $.
\end{proof}

We use relative importance relation and variable order interchangeably.  If a given acyclic CP-net $N$ induces the variable order $X\rhd Y$, we denote it as $N\models X\rhd Y$.  By using the directed acyclic graph of the CP-net, we can easily find all induced variable orders.  Note that the induced variable order is not necessarily transitive.

\begin{lemma} Given a totally dependent acyclic CP-net $N$, $N\models X\rhd Y$ and $N\models Y\rhd Z$ do not always imply $N\models X\rhd Z$.
\end{lemma}

\begin{proof} We can prove that by counter example when considering the following two cases.  First, $\vv{(X,Z)}$ exists in $N$.  Using Lemma~\ref{lemma_relative_importance_directed_arc}, $X\rhd Z$ holds, when $N\models X\rhd Y$ and $N\models Y\rhd Z$ hold.  The transitivity is true in this case.  Second, $\vv{(X,Z)}$ does not exist in $N$.  Using Lemma~\ref{lemma_relative_importance_directed_arc}, $X\rhd Z$ does not hold, while $N\models X\rhd Y$ and $N\models Y\rhd Z$ hold. The transitivity is not true in this case.
\end{proof}

\subsubsection{Outcome order}

An acyclic CP-net always represents a partial order of the preferences over the outcomes, i.e., the order is asymmetric and transitive.  In the following, we give a necessary and sufficient condition for the CP-net to be totally ordered, i.e., every two outcomes are preferentially comparable.

\begin{theorem} An acyclic CP-net represents a total order over the outcomes, if and only if the induced variable order is complete.
\label{Ch3OrderingTh1}
\end{theorem}

\begin{proof} ($\Rightarrow$)  We prove this by contradiction.  Let $N$ be a CP-net on $V$ which represents a total order over the outcomes and there exists no induced relative importance relation between any pair of two variables $X$ and $Y$. Using Lemma~\ref{lemma_relative_importance_directed_arc}, there is no arc between $X$ and $Y$.  Therefore, $X$ and $Y$ are preferentially independent given $w\in D(V-\{X,Y\})$.  Now, since there is no relative importance relation between $X$ and $Y$, there exist some $x_1,x_2\in D(X)$ and some $y_1,y_2\in D(Y)$ such that $x_1 y_2w$ and $x_2 y_1w$ are not comparable given that $x_1\succ x_2$ and $y_1\succ y_2$. Since these two outcomes are not comparable, the preference order over the outcomes cannot be complete.  This contradicts with the total order property.

($\Leftarrow$) Let $N$ be a CP-net on $V$ such that the CP-net induced variable order is complete.  Let $o_i$ and $o_j$ be any two outcomes of $N$, which differ on the values of $U\subseteq V$.  The variable order on $U$ is also complete, i.e., for every $X_1,X_2\in U$, either $X_1\rhd X_2$ or $X_2\rhd X_1$.  Given this completeness and the acyclicity property of the CP-net, there is one and only one variable $X\in U$ such that $X\rhd Y$ for every $Y\in U-\{X\}$, and this variable is selected.  Every parent of $X$ gives the same value to $o_i$ and $o_j$; otherwise the parent will exist in $U$ and the parent will be selected instead of $X$. Given the parents' value, $CPT(X)$ represents a preference order on $D(X)$. If X gives $x_i$ and $x_j$ to $o_i$ and $o_j$ correspondingly, we have $o_i\succ o_j$ if $x_i\succ x_j$ or $o_j\succ o_i$ if $x_j\succ x_i$.  Therefore, $o_i$ and $o_j$ are comparable.
\end{proof}

\subsection{Proposed model: The CPR-net}
\label{section_cpr_net}

The condition of Theorem~\ref{Ch3OrderingTh1} immediately motivates us to simply extend the CP-net model such that it always represents a total order over the outcomes.  This is achieved by eliciting Additional Relative Importance (ARI) statements between every pair of variables in which the CP-net does not induce a relative importance relation.  We call such a pair of variables a Non-Ordered Pair (NOP).  The extended model is formally defined below.

\begin{definition}~(CPR-net)~A CP-net with ARI statements (CPR-net) $N_r$ is an acyclic CP-net $N$ extended to the following:  for every NOP $\langle X,Y \rangle$ in $N$, we have either ARI $X\rhd Y$ or ARI $Y\rhd X$ in $N_r$, which is denoted as a dashed directed arc, $\vv{\langle X,Y\rangle}$ for $X\rhd Y$ and $\vv{\langle Y,X\rangle}$ for $Y\rhd X$, in the directed graph.
\end{definition}

Note that a CPR-net, excluding the ARI statements, is simply an acyclic CP-net.  This guarantees that the CPR-net preserves the conditional preference statements under the \emph{ceteris paribus} assumption described by the acyclic CP-net.  The construction procedure of the CPR-net is adopted from the construction procedure of the acyclic CP-net~\cite{Boutilier2004_CPnetJAIR}. After we construct the acyclic CP-net, we identify all the NOPs using the directed acyclic graph.  Then, we ask the user to provide an ARI statement for each of the NOPs.  For every ARI statement, we add the corresponding dashed directed arc with the CP-net graph.  Note that, nothing prevents the CPR-net to be cyclic, although the corresponding CP-net is acyclic.  However, in this paper, we consider only acyclic CPR-nets, while we leave cyclic CPR-nets for a future study.

\begin{example} In the CP-net of Figure~\ref{Label_totallydependentcpnet_fig}, $\langle A,C\rangle$ is a NOP.  We can easily extend this CP-net by adding the ARI statement $A\rhd C$ or $C\rhd A$.  We illustrate the CPR-net in Figure~\ref{Ch3CPRnet_fig} that indicates that the ARI $A\rhd C$ is represented by the dashed directed arc $\vv{\langle A,C\rangle}$.  This is an acyclic CPR-net. However, instead of the ARI $A\rhd C$, if we consider the ARI $C\rhd A$ and add the dashed directed arc $\vv{\langle C,A\rangle}$, we will get a cyclic CPR-net.
\end{example}

\begin{figure}[t]
\centerline{\includegraphics[width=0.5\textwidth]{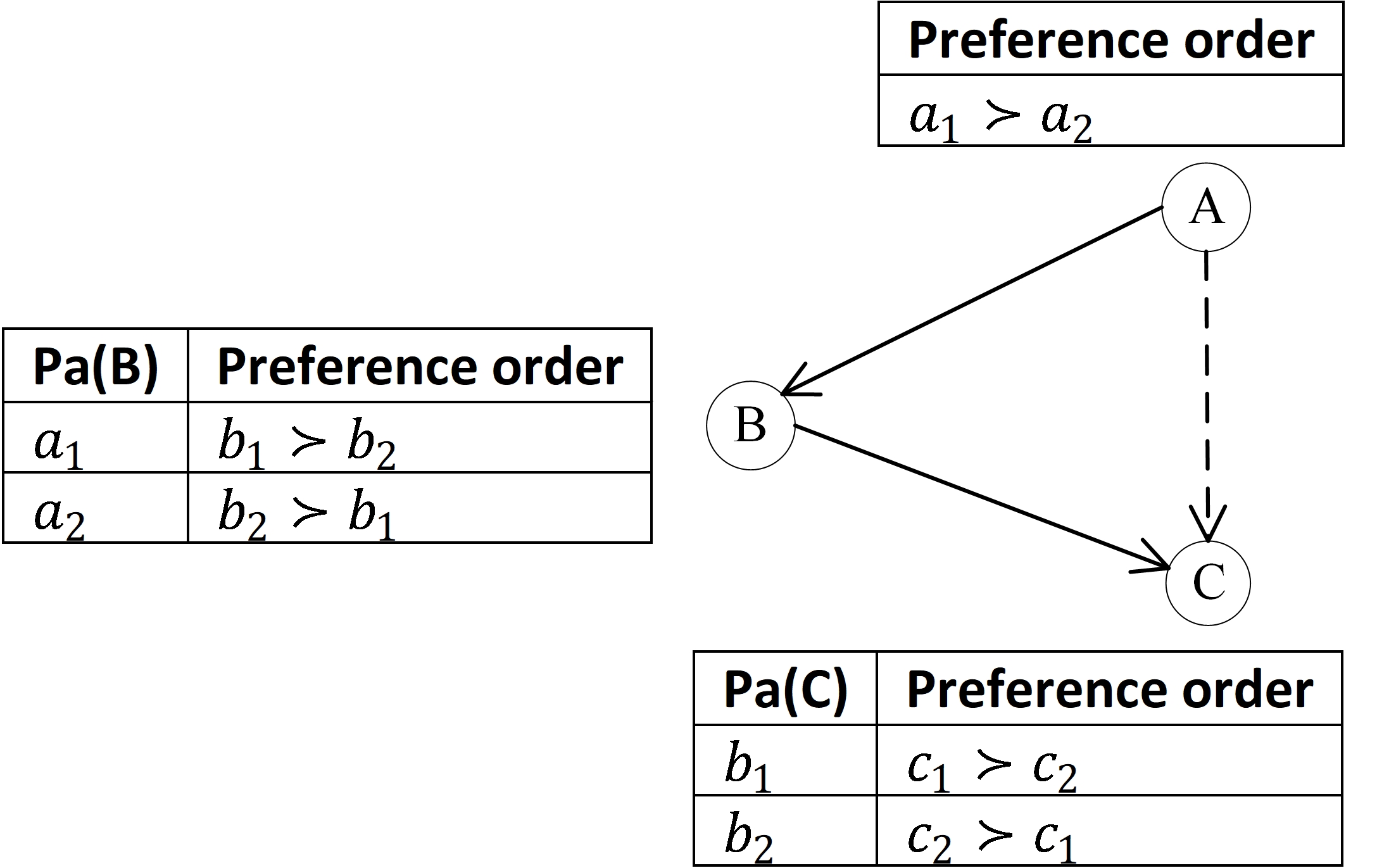}}
\caption{A CP-net with Additional Relative Importance Statements.}
\label{Ch3CPRnet_fig}
\end{figure}

The semantics of a CPR-net is defined below.  We use $\succ_p^X$ to denote the preference relation over $D(X)$ given $p\in D(Pa(X))$.

\begin{definition}~Consider a CPR-net $N_r$ on $V$.  Let $W=V-(\{X\}\cup Pa(X))$ and let $p\in D(Pa(X))$.  A preference order $\succ$ over $D(V)$ satisfies $\succ_p^X$ iff we have $x_ipw\succ x_jpw$, for every $w\in D(W)$, whenever $x_i \succ_p^X x_j$ holds.  The preference order $\succ$ satisfies $CPT(X)$ iff it satisfies $\succ_p^X$ for every $p\in D(Pa(X))$. The preference order $\succ$ satisfies an ARI $X\rhd Y$ iff for every $x_i,x_j\in D(X)$, every $y_a,y_b\in D(Y)$ and every $z\in D(Z)$ such that $Z=V-\{X,Y\}$, $\succ$ gives $x_iy_az\succ x_j y_b z$ whenever $x_i \succ_z^X x_j$ holds.  The preference order $\succ$ satisfies the CPR-net $N_r$ iff it satisfies every CPT and every ARI of $N_r$.
\end{definition}

\begin{example} Let the preference order $\succ$ over the outcomes of the CPR-net in Figure~\ref{Ch3CPRnet_fig} be $a_1 b_1 c_1\succ a_1 b_1 c_2\succ a_1 b_2 c_2\succ a_1 b_2 c_1\succ a_2 b_2 c_2\succ a_2 b_2 c_1\succ a_2 b_1 c_1\succ a_2 b_1 c_2$.  In $\succ$, between any two outcomes those differ on only the value of $A$, the outcome containing $a_1$ is preferred to the outcome containing $a_2$.  So, $\succ$ satisfies the preference order $a_1\succ a_2$ and thus the $CPT(A)$.  Given $A=a_1$ in $\succ$, between any two outcomes those differ on only the value of $B$, the outcome containing $b_1$ is preferred to the outcome containing $b_2$; and given $A=a_2$, between any two outcomes those differ on only the value of $B$, the outcome containing $b_2$ is preferred to the outcome containing $b_1$.  Therefore, $\succ$ satisfies $a_1:b_1\succ b_2$ and $a_2:b_2\succ b_1$; and thus $\succ$ satisfies the $CPT(B)$.  Similarly, we can show that $\succ$ satisfies the $CPT(C)$.  On the other hand, $\succ$ gives $a_1 b_1 c_2\succ a_2 b_1 c_1$ and $a_1 b_2 c_1\succ a_2 b_2 c_2$; which indicates that $\succ$ satisfies the ARI $A\rhd C$.  Therefore, we say that $\succ$ satisfies the CPR-net. $\hfill \square$
\end{example}

Surprisingly, we find that an acyclic CPR-net always represents a total order of the preferences over the outcomes.  This is explored using the theorem below.

\begin{theorem} There is one and only one preference order over the outcomes of an acyclic CPR-net, which satisfies the acyclic CPR-net.
\end{theorem}

\begin{proof} Let the CPR-net $N_r$ be formed from the acyclic CP-net $N$ such that, for every NOP $\langle X,Y\rangle$ in N, we have either ARI $X\rhd Y$ or ARI $Y\rhd X$ in $N_r$.  This indicates that, for every pair of variables in $N_r$, we have a relative importance relation, either induced by $N$ or an ARI statement.  Using Theorem~\ref{Ch3OrderingTh1}, we have a preference order over the outcomes, which is a total order. The total order is also unique.
\end{proof}

\subsection{Constrained optimization with CPR-nets: the Constrained CPR-net model}
\label{section_constraint}

The main advantage of the CPR-net is when constraints are involved, which will lead to a constrained optimization problem. In this regard, we extend the CPR-net to a set of hard constraints, and we call the new model the Constrained CPR-net. Since an acyclic CPR-net represents a total order over the outcomes, the Constrained CPR-net has one single optimal outcome, where the optimal outcome is the one that is feasible with respect to the underlying CSP and is preferred to every other feasible outcome. In this section, we propose an algorithm that we call Search-CPR, to find the optimal outcome for a Constrained CPR-net.

Search-CPR is a recursive algorithm that takes three parameters as input in each call.  The first parameter $N_r$ is a sub CPR-net, which is initially the original CPR-net $N_{r_{orig}}$.  The second parameter $C$ is a set of hard constraints, which is initially the original set of constraints $C_{orig}$.  The third parameter $K$ is an assignment to the variables of $N_{r_{orig}}-N_r$, which is initially $null$.  Search-CPR produces the optimal outcome which is saved in $S$.

\begin{algorithm}\textbf{Algorithm 1.}~Search-CPR($N_{r},C,K$)\\
\textbf{Input:} Acyclic CPR-net $N_r$, constraints $C$, assignment $K$ to $N_{r_{orig}}-N_r$

\textbf{Output:} Optimal outcome $S$ with respect to $C$ and $N_r$

\begin{algorithmic}[1]
  \STATE Choose a variable $X$ with no parents in $N_r$
  \STATE Let $x_1\succ x_2\succ \cdots \succ x_m$ be the preference order over $D(X)$ given the assignment to $Pa(X)$ in $K$

  \FOR {$i=1;i\leq m;i=i+1$}
    \STATE Strengthen the constraints $C$ by $X=x_i$ to obtain $C_i$
    \IF {$C_i$ is inconsistent}
      \STATE \textbf{continue} with the next iteration
    \ELSE
      \STATE Let $K'$ be the partial assignment induced by $X=x_i$ and $C_i$
      \STATE Reduce $N_r$ to $N_{r_i}$ by removing each variable that is instantiated in $K'$ and by restricting the CPT of the remaining variables to $K'$
      \IF {$N_{r_i}$ is empty}
        \STATE $S\leftarrow K\cup K'$
        \STATE \bf{Exit}
      \ELSE
        \STATE Search-CPR($N_{r_i},C_i,K\cup K'$)
      \ENDIF
    \ENDIF
  \ENDFOR
  \STATE $S\leftarrow null$
  \end{algorithmic}
\end{algorithm}

In line 1, Search-CPR chooses a variable $X$ with no parents from the sub CPR-net. The acyclicity property guarantees that such a variable exists.  Lines 4-16 continue for every value $x_i$ of $X$ with the most preferred one first, and so on, given $Pa(X)$'s instantiation in $K$. The constraints $C$ are strengthened using $X=x_i$ in line 4 to get the current set of constrains $C_i$.  If the constraints are inconsistent, the loop continues with the next value.  Otherwise, in line 8, the partial assignment $K'$ induced by $X=x_i$ and $C_i$ is determined.  In line 9, the sub CPR-net $N_{r_i}$ is obtained by removing all instantiated variables from $N_r$, and by restricting the CPTs of the remaining variables to the current instantiation which is $K\cup K'$.  Line 10 checks the termination criterion.  If all variables are instantiated, the optimal outcome is saved in $S$ and Search-CPR is terminated; otherwise $K'$ is forwarded with $K$ to the next call of Search-CPR using line 14.  Finally, if no feasible outcome exists, Search-CPR saves $null$ in $S$ using line 18.

\begin{example} Let us apply Search-CPR for the CPR-net $N_{r_{orig}}$ in Figure~\ref{label_search_cpr_illustration_fig}~(a) with the set of four binary constrains $C_{orig}=\{\{A=a_1\}\rightarrow \{B=b_2\},\{A=a_1\}\leftrightarrow \{D=d_2\},\{C=c_2\}\rightarrow \{B=b_1\},\{C=c_1\}\leftrightarrow \{D=d_1\}\}$.  Initially, we have that $K$ is $null$.  The corresponding search tree is illustrated in Figure~\ref{label_search_cpr_Search_tree_fig}.

\begin{figure}[htbp]
\centerline{\includegraphics[width=1.0\textwidth]{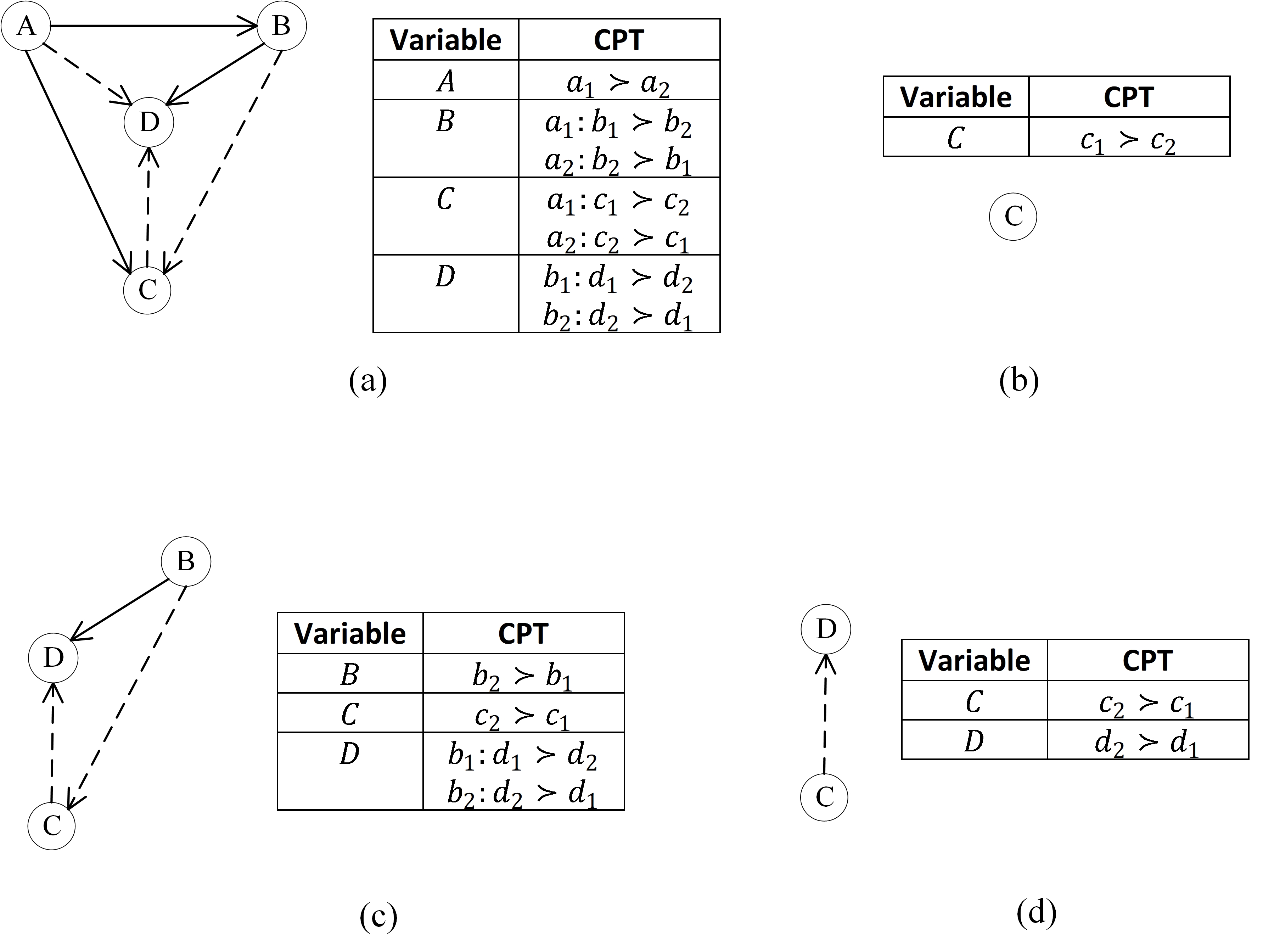}}
\caption{Illustration of the Search-CPR algorithm.}
\label{label_search_cpr_illustration_fig}
\end{figure}

\begin{figure}[htbp]
\centerline{\includegraphics[width=0.9\textwidth]{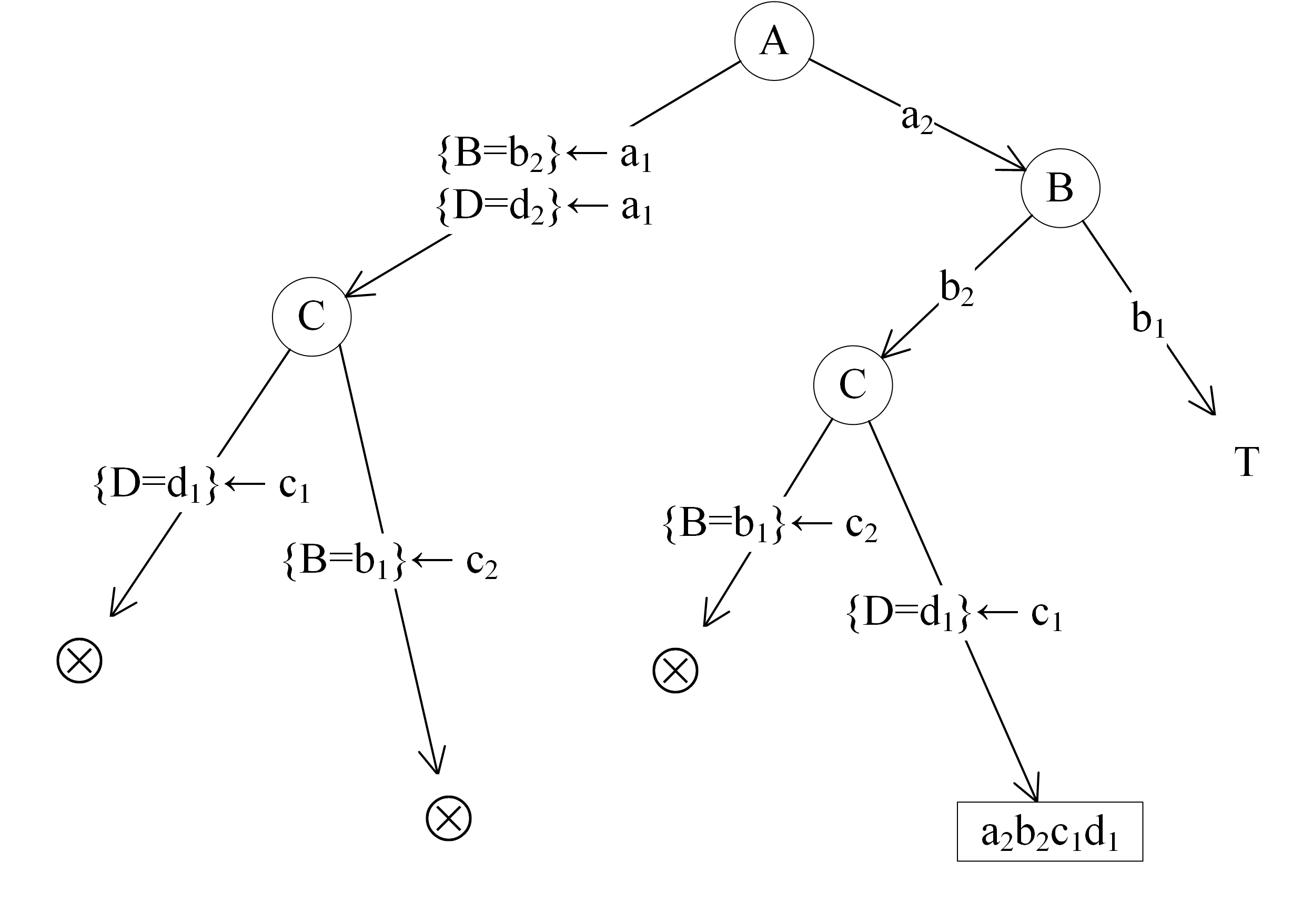}}
\caption{The search tree of the Search-CPR algorithm for the CPR-net in Figure~\ref{label_search_cpr_illustration_fig}~(a).}
\label{label_search_cpr_Search_tree_fig}
\end{figure}

In the initial call (call 0), variable $A$ is chosen.  Two branches are initiated, one for $A=a_1$ and another for $A=a_2$.  For branch $A=a_1$, after strengthening $C_{orig}$ with $A=a_1$, we get $C_{a_1}=\{\{B=b_2\},\{D=d_2\},\{C=c_2\}\rightarrow \{B=b_1\},\{C=c_1\}\leftrightarrow \{D=d_1\}\}$\footnote{At this point, $A$, $B$ and $D$ are instantiated. The only remaining variable is $C$ which can be assigned to $c_1$ or $c_2$.  Note that, for $C=c_1$, $C_{a_1}$ implies $D=d_1$, which is inconsistent with the partial assignment $K_{a_1}'=a_1 b_2 d_2$.  Again, for $C=c_2$, $C_{a_1}$ implies $B=b_1$, which is also inconsistent with $K_{a_1}'$.  Therefore, the branch $A=a_1$ does not need to be extended. This observation is a clear indication of the forward checking method~\cite{Dechter2003constraint}.  In this way, we can always incorporate various constraint propagation techniques with our algorithm.}. $C_{a_1}$ is not inconsistent and this iteration continues.   The partial assignment in line 8 is $K_{a_1}'=a_1 b_2 d_2$.  The reduced sub CPR-net $N_{r_{a_1}}$ is obtained by removing $A$, $B$ and $D$ from $N_{r_{orig}}$, which is shown in Figure~\ref{label_search_cpr_illustration_fig}~(b).  Note that, the CPT of $C$ is restricted to $A=a_1$. Since $N_{r_{a_1}}$ is not empty, the termination criterion is not fulfilled.  Search-CPR is again called with $N_{r_{a_1}}$, $C_{a_1}$ and $K_{a_1}=a_1 b_2 d_2$ (call 1).  In this call, we have one variable $C$, and two branches are created for $C=c_1$ and $C=c_2$ correspondingly.  For the branch $C=c_1$, after strengthening $C_{a_1}$ by $C=c_1$, we get $C_{a_1 c_1}=\{\{B=b_2\},\{D=d_2\},\{C=c_2\}\rightarrow \{B=b_1 \},\{D=d_1 \}\}$, which indicates an inconsistency with $K_{a_1}$ for the assignment of $D$.  Therefore, the branch stops, and we denote it as $\otimes$ in Figure~\ref{label_search_cpr_Search_tree_fig}.  For the branch $C=c_2$, after strengthening $C_{a_1}$ by $C=c_2$, we get $C_{a_1 c_2}=\{\{B=b_2 \},\{D=d_2\},\{B=b_1 \},\{C=c_1 \}\leftrightarrow \{D=d_1 \}\}$, which is inconsistent for the assignment of $B$.  Therefore, this branch also stops.  There is no more remaining iteration for values of $C$, and the loop ends. Call 1 is terminated.

Now in call 0, we continue the branch for $A=a_2$.  Strengthening $C_{orig}$ with $A=a_2$ remains the same, i.e., $C_{a_2}=C_{orig}$.  In line 8, we get the partial assignment $K_{a_2}'=a_2$.  The reduced CPR-net $N_{r_{a_2}}$ after removing $A$ from $N_{r_{orig}}$ is shown in Figure~\ref{label_search_cpr_illustration_fig}~(c).  Search-CPR is called with $N_{r_{a_2}}$, $C_{a_2}$ and $K_{a_2}=a_2$ (call 2).  For this call, variable $B$ is chosen.  Given $A=a_2$ in $K_{a_2}$, the preference order over $D(B)$ is $b_2\succ b_1$ as we see in Figure~\ref{label_search_cpr_illustration_fig}~(c).  Let us continue with the branch $B=b_2$.  Strengthening $C_{a_2}$ with $B=b_2$ remains the same, i.e., $C_{a_2 b_2}=C_{a_2}$.  We have the partial assignment $K_{a_2 b_2}'=b_2$.  The reduced CPR-net $N_{r_{a_2 b_2}}$ after removing $B$ from $N_{r_{a_2}}$ is shown in Figure~\ref{label_search_cpr_illustration_fig}~(d).  Search-CPR is called with $N_{r_{a_2 b_2}}$, $C_{a_2 b_2}$ and $K_{a_2 b_2}=a_2,b_2$ (call 3).  In this call, variable $C$ is chosen.  The branch for $C=c_2$ stops as the set of constraints after strengthening with $C=c_2$ implies $B=b_1$, which contradicts with $K_{a_2 b_2}$ for the assignment of $B$.  Let us consider the branch $C=c_1$.  After strengthening $C_{a_2 b_2}$ with $C=c_1$, we get: $C_{a_2 b_2 c_1}=\{\{A=a_1\}\rightarrow \{B=b_2 \},\{A=a_1 \}\leftrightarrow \{D=d_2\},\{C=c_2 \}\rightarrow \{B=b_1 \},\{D=d_1 \}\}$.  This set of constraints is not inconsistent and the iteration continues.  In line 8, we get the partial assignment, $K_{a_2 b_2 c_1}'=c_1,d_1$.  The sub CPR-net after removing $C$ and $D$ from $N_{r_{a_2 b_2}}$ becomes $empty$.  The termination criterion in line 10 is fulfilled, which indicates that Search-CPR reaches to the solution.  The rest of the search tree is not needed to search, and is indicated as $T$ in Figure~\ref{label_search_cpr_Search_tree_fig}. The feasible optimal outcome $a_2 b_2 c_1 d_1$ is saved in $S$, and Search-CPR is terminated. $\hfill \square$
\end{example}

We present the correctness of Search-CPR using the theorem below.

\begin{theorem} Let $N_r$ and $C$ be a CPR-net and a set of hard constraints on $V$ correspondingly.  If Search-CPR produces the outcome $o_i$, then $o_i$ is the optimal outcome with respect to $N_r$ and $C$.
\end{theorem}

\begin{proof} In order to prove the theorem, we have to prove that: (1) every outcome $o_j$ such that $o_j\succ o_i$ is infeasible; and (2) every outcome $o_k$ such that $o_i\succ o_k$ is pruned.  We prove these below.

(1)  This proof is by contradiction.  Let us consider that Search-CPR produces $o_i$.  Let $o_j$ be a feasible outcome and we have $o_j\succ o_i$.  Let $o_j$ and $o_i$ differ on the values of $U\subseteq V$.  There is a variable $X\in U$ such that for every $Y\in U-\{X\}$, we have $X\rhd Y$.  This shows that no parent of $X$ is in $U$ and so every parent of $X$ gives the same value to $o_j$ and $o_i$.  Let $Pa(X)$ gives $p$ to $o_j$ and $o_i$.  Let $X$ gives $x_j$ and $x_i$ to $o_j$ and $o_i$ correspondingly.  $o_j\succ o_i$ implies that $x_j\succ x_i$ given $Pa(X)$'s instantiation to $p$.  Using lines 2-3, Search-CPR produces the branch for $X=x_j$ before the branch for $X=x_i$, in the search tree.  The outcome $o_j$ is located with the branch $X=x_j$.  Since, $o_j$ is feasible, during the search process for the branch $X=x_j$, Search-CPR determines $o_j$ as the feasible optimal outcome and Search-CPR terminates.  Therefore, $o_i$ is not produced by Search-CPR.  Thus, $o_j$ cannot be feasible.

(2) As soon as Search-CPR finds $o_i$, it terminates, i.e., no other outcome is searched.  Let $o_k$ be such an outcome that is not searched by Search-CPR.  Let $o_i$ and $o_k$ differ on the values of $U\subseteq V$.  There is a variable $X\in U$ such that for every $Y\in U-\{X\}$, we have $X\rhd Y$.  It gives us that no parent of $X$ is in $U$ and so every parent of $X$ gives the same value to $o_i$ and $o_k$.  Let $Pa(X)$ gives $p$ to $o_i$ and $o_k$.  Let $X$ gives $x_i$ and $x_k$ to $o_i$ and $o_k$ correspondingly.  Since $o_i$ is obtained prior to searching for $o_k$, using lines 2-3, it is clear that $x_i\succ x_k$ given $Pa(X)$'s instantiation to $p$, which implies $o_i\succ o_k$.
\end{proof}

It is hard to compare the complexity of a decision problem with the complexity of an optimization problem. However, in some sense, we can say that the constrained optimization problem, based on the preferences described by a CPR-net, is no harder than the underlying CSP. On the other hand, the performance of a CSP-solving algorithm is often influenced by the order in which the variables are instantiated~\cite{Mouhoub2011VariableValueOrderingCSP,Yong2017VariableValueOrderCSP}. In this case, Search-CPR is restricted to use the topological order of the CPR-net as the instantiation order, which might not be optimal for solving the CSP efficiently. Nevertheless, the CPR-net based constrained optimization has the following advantage over CSPs.  The real world CSPs usually have a set of feasible solutions, and the solving algorithms need to find all such solutions.  After getting all feasible solutions, the decision maker needs to do a trade-off analysis to choose the best one~\cite{Utyuzhnikov20096Pareto_Set}.  Using our method, a trade-off analysis is not needed as the method returns a single solution.

\section{Constrained LP-trees}
\label{SEC_CON_LP_TREE}
In this section, we extend the LP-tree to a set of hard constraints. We call the new model the Constrained LP-tree. Then, we propose a solving algorithm, called Search-LP, that returns the most preferable feasible outcome of the Constrained LP-tree. Finally, we discuss formal properties of Search-LP.

\subsection{Compatible LP-trees and Reduced LP-trees}

In Search-LP algorithm, every time a set of variables are instantiated, the original LP-tree is updated to a sub LP-tree by removing the instantiated variables. Then, the algorithm is called recursively with the sub LP-tree until it gets the solution. In this regard, first, we define the compatibility between the LP-tree and the sub LP-tree. We call that the sub LP-tree is compatible to the LP-tree if the preferences represented by the sub LP-tree are also represented by the LP-tree given the instantiation to the instantiated variables. Then, given an LP-tree and some instantiation to a subset of variables, we explain a method to reduce the LP-tree by eliminating the instantiated variables.  We prove that the Reduced LP-tree is always a Compatible LP-tree.

\begin{definition} Let $L$ be an LP-tree on $V$, and $L_c$ be an LP-tree on $U\subset V$.  We call that $L_c$ is compatible to $L$ if there exists some instantiation $w$ of $V-U$ such that for every two outcomes, $o_{1c}$ and $o_{2c}$, of $L_c$, we have: $(L_c\models o_{1c}\succ o_{2c})\Rightarrow (L\models w o_{1c}\succ w o_{2c})$.
\end{definition}

\begin{figure}[t]
\centering
\includegraphics[width=0.2\textwidth]{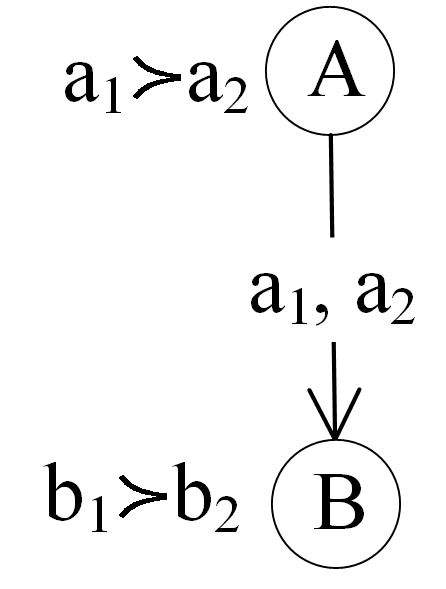}
\caption{A Compatible LP-tree of the LP-tree in Figure~\ref{fig_my_dinner_lp_tree} given $C=c_1$.}
\label{fig_compatible_lp_tree000}
\end{figure}

\begin{example} Let us assume that Figure~\ref{fig_compatible_lp_tree000} represents an LP-tree on $\{A,B\}$. We get the induced preference order as:  $a_1 b_1 \succ a_1 b_2 \succ a_2 b_1\succ a_2 b_2$. We now check if Figure~\ref{fig_compatible_lp_tree000} is compatible to Figure~\ref{fig_my_dinner_lp_tree}. Let $C$ be instantiated to $c_1$ in Figure~\ref{fig_my_dinner_lp_tree}. From Example~\ref{example_my_dinner_lp_tree_total_order}, given $C=c_1$, we get $a_1 b_1 \succ a_1 b_2 \succ a_2 b_1\succ a_2 b_2$ which is exactly the one represented by Figure~\ref{fig_compatible_lp_tree000}. Therefore, Figure~\ref{fig_compatible_lp_tree000} is a Compatible LP-tree of the LP-tree in Figure~\ref{fig_my_dinner_lp_tree} for $C=c_1$. $\hfill \square$
\end{example}

We now explain how to get a Reduced LP-tree from an LP-tree.

\begin{definition} Let $L$ be an LP-tree over $V$.  Let $W\subset V$ be instantiated to $w$. An LP-tree $L_r$ on $U=V-W$ is called Reduced LP-tree of $L$ given $W=w$ if $L_r$ is obtained from $L$ in the following way.  For every node $X\in W$ where $X$ is instantiated to $x$, we apply the following for every existence of $X$~\footnote{Note that a variable $X$ might exist in an LP-tree more than once. For example, in Figure~\ref{fig_my_dinner_lp_tree}, variable $B$ exists twice.}:
\begin{enumerate}
	\item All sub-trees for $X$ except the subtree $X=x$ are deleted, and then the $CPT(Y)$ for every $Y\in De(X)$ is restricted to $X=x$.
	\item If $X$ has both a child and a parent, the arc $\vv{(Pa(X),X)}$ is replaced with\\$\vv{(Pa(X),Ch(X))}$, and then the arc $\vv{(X,Ch(X))}$ is deleted.  If $X$ has a parent but not a child, the arc $\vv{(Pa(X),X)}$ is deleted.  If $X$ has a child but not a parent, the arc $\vv{(X,Ch(X))}$ is deleted.
	\item The node $X$ and its CPT are deleted.
	\item If there exist identical branches for any ancestor variable of $X$, then the branches are merged together.
\end{enumerate}
\label{definitin_compatible_lp_tree}
\end{definition}

\begin{example} Consider the variable $B$ is instantiated to $b_1$ for the LP-tree of Figure~\ref{fig_my_dinner_lp_tree}, and we want to find the Reduced LP-tree.  The node $B$ exists twice in Figure~\ref{fig_my_dinner_lp_tree}, first in branch $A=a_1$ and then in branch $A=a_2$.

First, consider $B$ in branch $A=a_1$.  Since $B$ is a parent node, the label $b_1,b_2$ for the arc $\vv{(B,C)}$ is updated to $b_1$ and the $CPT(C)$ is restricted to $B=b_1$.  Since $B$ has both child and parent, the arc $\vv{(A,B)}$ is replaced with $\vv{(A,C)}$, and the arc $\vv{(B,C)}$ is deleted.  Then, node $B$ and $CPT(B)$ are deleted.

Second, consider $B$ in branch $A=a_2$.  Since $B$ has a parent but not a child, the arc $\vv{(C,B)}$ is deleted.  Then, node $B$ and $CPT(B)$ are deleted.  At this point, node $A$ has two branches, one for $a_1$ and another for $a_2$.  These branches are identical, and are merged together with label $a_1,a_2$ for the arc $\vv{(A,C)}$.  The Reduced LP-tree is shown in Figure~\ref{fig_compatible_lp_tree}. $\hfill \square$
\end{example}

\begin{figure}[t]
\centering
\includegraphics[width=0.2\textwidth]{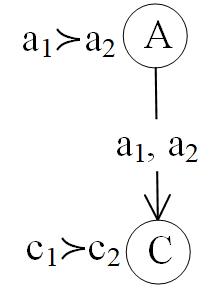}
\caption{A Reduced LP-tree of the LP-tree in Figure~\ref{fig_my_dinner_lp_tree} given $B=b_1$.}
\label{fig_compatible_lp_tree}
\end{figure}

\begin{lemma} Let $L$ be an LP-tree on $V$ and $L_r$ be a Reduced LP-tree of $L$ on $U\subset V$.  Then, $L_r$ is also a Compatible LP-tree of $L$.
\label{lemma_compatible_lp_tree_correctness}
\end{lemma}

\begin{proof} Let $L_r$ be obtained from $L$ for $(V-U)$'s instantiation to $p$ in $L$. We need to show that the preference order induced by $L_r$ is represented by the preference order induced by $L$ for $(V-U)=p$. Let $o_{1c}$ and $o_{2c}$ be two outcomes of $L_r$, and $L_r$ induces $o_{1c}\succ o_{2c}$.  Let the variable $X$ gives two different values $x_1$ and $x_2$ to $o_{1c}$ and $o_{2c}$ correspondingly, and $An(X)$ in $L_r$, denoted as $A_{L_r}$, gives the same value $z$ to $o_{1c}$ and $o_{2c}$.  By Definition~\ref{definition_lex_order}, we can clearly see that $o_{1c}\succ o_{2c}$ implies $x_1\succ x_2$ given $A_{L_r}=z$.

The definition of Reduced LP-tree indicates that $An(X)$ in $L$ is a subset of, or equals to, $A_{L_r}\cup (V-U)$.  Let $A_{L_r}\cup (V-U)$ is instantiated to $q$.  Therefore, $An(X)$ in $L$ gives the same value (let $r$) to the outcomes $po_{1c}$ and $po_{2c}$ of $L$.  Given $An(X)=r$ in $L$, we have $x_1\succ x_2$ which implies $po_{1c}\succ po_{2c}$.  Thus, $o_{1c}\succ o_{2c}$ induced by $L_r$ is also held by $L$ for $V-U=p$.
\end{proof}

\begin{example} The preference order induced by the Reduced LP-tree of Figure~\ref{fig_compatible_lp_tree} is $a_1 c_1\succ a_1 c_2\succ a_2 c_1\succ a_2 c_2$.  This order is held by the LP-tree of Figure~\ref{fig_my_dinner_lp_tree} for $B=b_1$, i.e., $a_1 b_1 c_1\succ a_1 b_1 c_2\succ a_2 b_1 c_1\succ a_2 b_1 c_2$ is induced by Figure~\ref{fig_my_dinner_lp_tree} (see Example~\ref{example_my_dinner_lp_tree_total_order}). $\hfill \square$
\end{example}

\subsection{Search-LP algorithm}

\begin{algorithm}\textbf{Algorithm 2.}~Search-LP($L_r,C,K$)\\
\textbf{Input:} The original LP-tree $L_{orig}$, the original set of constraints $C_{orig}$

\textbf{Output:} The most preferable feasible outcome with respect to $L_{orig}$ and $C_{orig}$

\textbf{Initialization:} $L_r=L_{orig}$,  $C=C_{orig}$ and $K=null$

\begin{algorithmic}[1]
  \STATE Let $X$ be the root node of $L_r$
  \STATE Let $x_1\succ x_2\succ \cdots \succ x_l$ be the preference order on $D(X)$
  \FOR{$i=1;i\leq l;i=i+1$}
    \STATE Strengthen $C$ with $X=x_i$ to get $C_i$
    \IF{$C_i$ is inconsistent}
        \STATE \textbf{continue} with the next iteration
    \ENDIF
	\STATE Let $K_i$ be the partial assignment induced by $X=x_i$ and $C_i$
	\STATE Find the Reduced LP-tree $L_{r_i}$ given the instantiation $K_i$
    \IF{$L_{r_i}$ is empty}
        \STATE Set: $S=K\cup K_i$
        \STATE \textbf{exit}
    \ELSE
    	\STATE Search-LP($L_{r_i},C_i,K\cup K_i$)
    \ENDIF
  \ENDFOR
  \STATE Set: $S=null$
  \STATE \textbf{end}
\end{algorithmic}
\end{algorithm}

Given a Constrained LP-tree, the optimal outcome for the corresponding LP-tree might be infeasible. In this case, our goal is to find the most preferable feasible outcome. The problem is not trivial as the underlying CSP is NP-complete, in general~\cite{Dechter2003constraint}.  For this purpose, we propose a recursive backtrack search algorithm that we call Search-LP.  Search-LP has three parameters.  The first parameter is a Reduced LP-tree $L_r$ which is initially the original LP-tree $L_{orig}$.  The second parameter is a set of hard constraints $C$ which is initially the original set of constraints $C_{orig}$.  The third parameter is a partial assignment $K$ on $L_{orig}-L_r$ which is initially $null$.  Search-LP returns the most preferable feasible outcome with respect to $L_{orig}$ and $C_{orig}$.

Search-LP begins with the root node $X$ and the preference order on $D(X)$.  For every value $x_i$ of $X$, the loop (lines 3-16) continues according to the preference order.  In line 4, the current set of constraints $C$ is strengthened by $X=x_i$ to get the new set of constraints $C_i$.  If $C_i$ is inconsistent, the branch $X=x_i$ is terminated (lines 5-7), and Search-LP continues with the next branch.  Otherwise, in line 8, the partial assignment $K_i$ induced by $X=x_i$ and $C_i$ is obtained.  In line 9, the Reduced LP-tree $L_{r_i}$ is obtained for the LP-tree $L_r$ and instantiation $K_i$, using Definition~\ref{definitin_compatible_lp_tree}.  If $L_{r_i}$ is empty, this indicates that all variables have been successfully instantiated.  The outcome is stored in $S$, and Search-LP is stopped (line 10-12).  If this termination criterion is not met, Search-LP is called recursively until the criterion is met (line 14).  Every time, the partial assignment $K_i$ is forwarded to the next call.  Finally, if no feasible outcome exists, i.e., the underlying CSP is inconsistent, Search-LP saves $null$ in $S$ (line 17) and ends (line 18).

\subsection{Search-LP: Example}

Consider the LP-tree $L_{orig}$ in Figure~\ref{fig_example_lp_tree}, which has four variables $A$, $B$, $C$ and $D$ with corresponding domains $\{a_1,a_2,a_3\}$, $\{b_1,b_2,b_3\}$, $\{c_1,c_2\}$ and $\{d_1,d_2\}$.  The initial set of constraints $C_{orig}$ is $\{\{A=a_1\}\rightarrow \{C=c_2\},\{B=b_1\}\rightarrow \{D=d_2\},\{C=c_2\}\leftrightarrow \{D=d_1\},\{B=b_3\}\rightarrow \{D=d_2\},\{D=d_2\}\rightarrow \{A=a_1\},\{A=a_2\}\leftrightarrow \{B=b_2\},\{A=a_3\}\leftrightarrow \{B=b_1\},\{A=a_3\}\rightarrow \{C=c_2\}\}$. We call Search-LP where initially the partial assignment $K$ is $null$.  For this example, the corresponding search tree is illustrated in Figure~\ref{fig_search_lp_seatch_tree}.

\begin{figure*}
\centering
\includegraphics[width=1.0\textwidth]{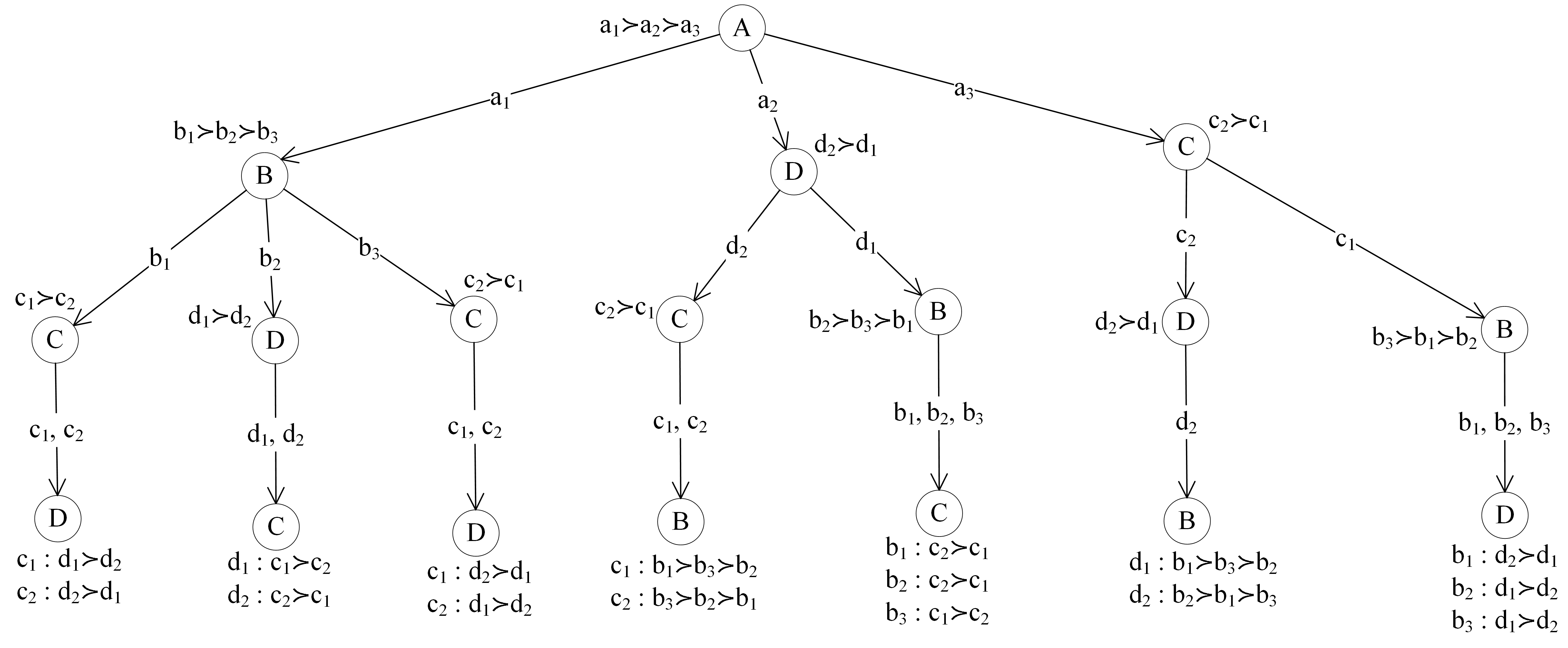}
\caption{An example of an LP-tree.}
\label{fig_example_lp_tree}
\end{figure*}

\begin{figure*}
\centering
\includegraphics[width=1\textwidth]{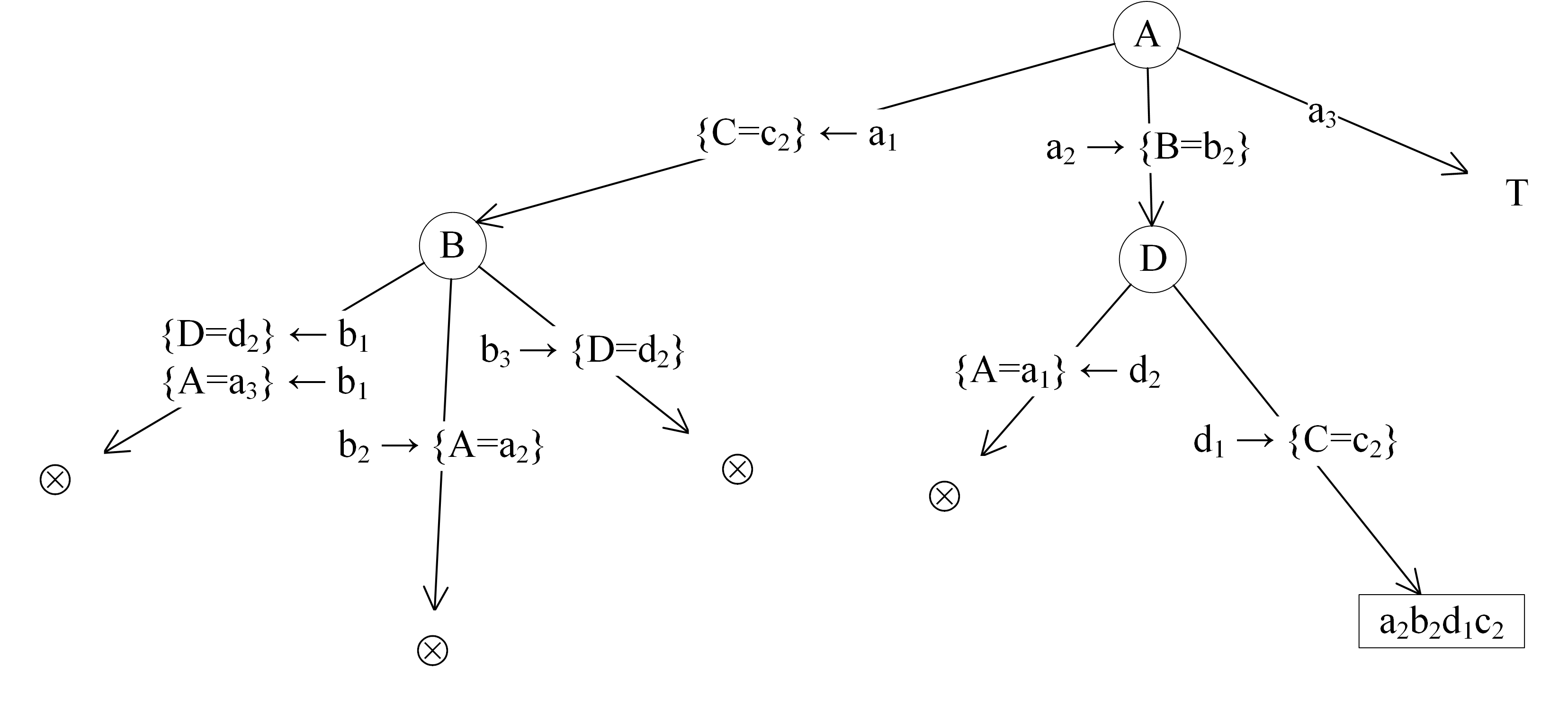}
\caption{Search tree of the Search-LP algorithm for the LP-tree of Figure~\ref{fig_example_lp_tree} with the hard constraints $C_{orig}=\{\{A=a_1\}\rightarrow \{C=c_2\},\{B=b_1\}\rightarrow \{D=d_2\},\{C=c_2\}\leftrightarrow \{D=d_1\},\{B=b_3\}\rightarrow \{D=d_2\},\{D=d_2\}\rightarrow \{A=a_1\},\{A=a_2\}\leftrightarrow \{B=b_2\},\{A=a_3\}\leftrightarrow \{B=b_1\},\{A=a_3\}\rightarrow \{C=c_2\}\}$.}
\label{fig_search_lp_seatch_tree}
\end{figure*}

\begin{figure*}[!h]
\centering
\includegraphics[width=0.5\textwidth]{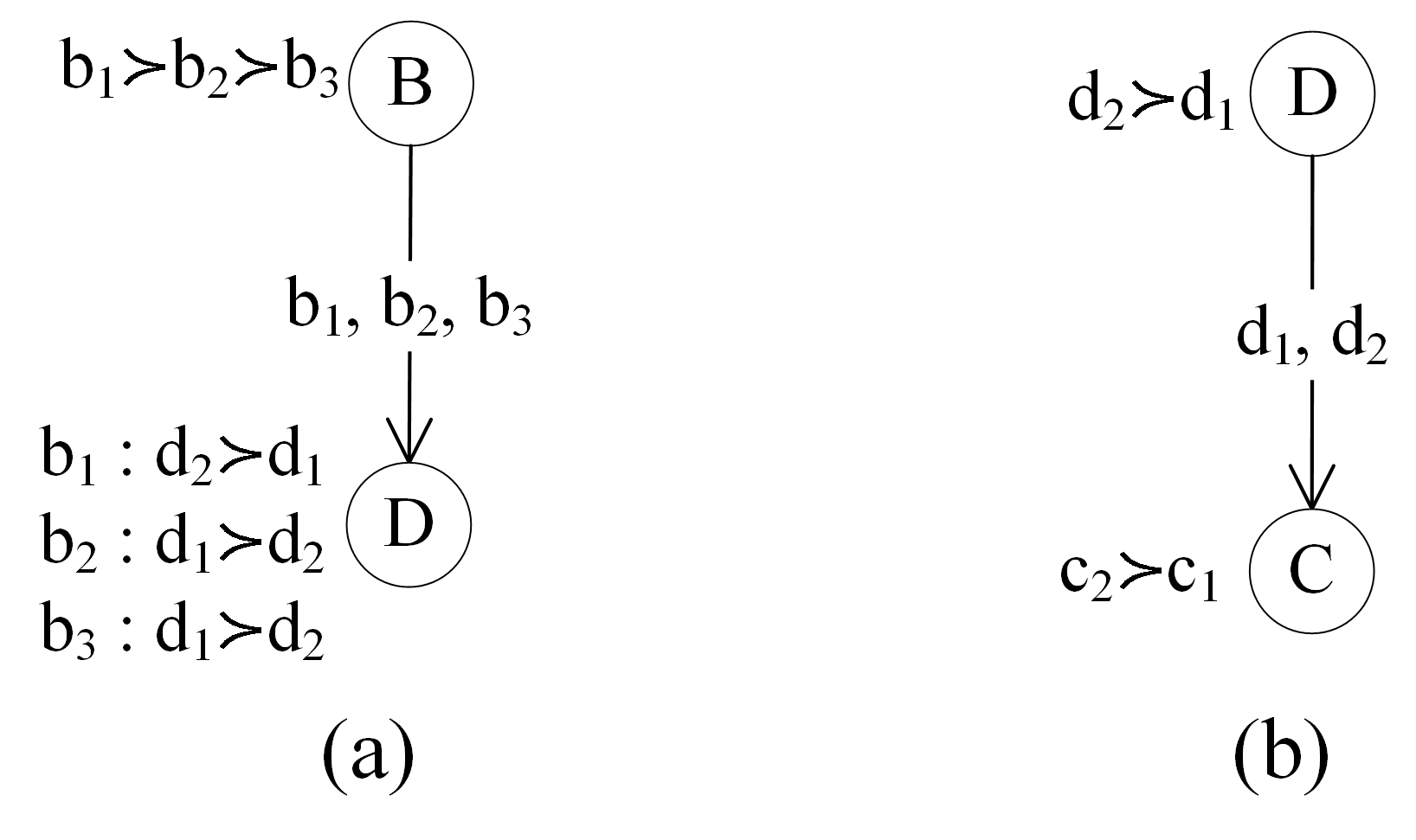}
\caption{Two Reduced LP-trees generated by the Search-LP algorithm for the LP-tree of Figure~\ref{fig_example_lp_tree} with the hard constraints $C_{orig}=\{\{A=a_1\}\rightarrow \{C=c_2\},\{B=b_1\}\rightarrow \{D=d_2\},\{C=c_2\}\leftrightarrow \{D=d_1\},\{B=b_3\}\rightarrow \{D=d_2\},\{D=d_2\}\rightarrow \{A=a_1\},\{A=a_2\}\leftrightarrow \{B=b_2\},\{A=a_3\}\leftrightarrow \{B=b_1\},\{A=a_3\}\rightarrow \{C=c_2\}\}$.}
\label{fig_search_lp_compatible_lp_trees}
\end{figure*}

In the initial call (call 0), the root node is $A$ which has the preference $a_1\succ a_2\succ a_3$ on its domain.  Let us continue the loop (lines 3-16) for $A=a_1$.  The corresponding branch in Figure~\ref{fig_search_lp_seatch_tree} is indicated with $A=a_1$.  In line 4, after strengthening $C_{orig}$ with $A=a_1$, we get $C_{a_1}=\{\{C=c_2\},\{B=b_1\}\rightarrow \{D=d_2\},\{C=c_2\}\leftrightarrow \{D=d_1\},\{B=b_3\}\rightarrow \{D=d_2\},\ignore{\{D=d_2\}\rightarrow \{A=a_1\},}\{A=a_2\}\leftrightarrow \{B=b_2\},\{A=a_3\}\leftrightarrow \{B=b_1\},\{A=a_3\}\rightarrow \{C=c_2\}\}$ which is consistent; so lines 5-7 are not relevant.  In line 8, the partial assignment induced by $A=a_1$ and $C_{a_1}$ is $K_{a_1}=a_1 c_2$ which indicates that variables $A$ and $C$ are instantiated at this point.  We build the Reduced LP-tree $L_{a_1 c_2}$ given the instantiation of $A$ and $C$.  $L_{a_1 c_2}$ is shown in Figure~\ref{fig_search_lp_compatible_lp_trees}~(a).  Since $L_{a_1 c_2}$ is not empty, Search-LP($L_{a_1 c_2},C_{a_1},a_1 c_2$) is called in line 14 (call 1).

For call 1, $B$ is the root node and its domain preference is $b_1\succ b_2\succ b_3$.  Consider the loop (lines 3-16) for $B=b_1$.  After strengthening $C_{a_1}$ with $B=b_1$ in line 4, we get $C_{a_1 b_1}=\{\{C=c_2\},\{D=d_2\},\{C=c_2\}\leftrightarrow \{D=d_1\},\{B=b_3\}\rightarrow \{D=d_2\},\ignore{\{D=d_2\}\rightarrow \{A=a_1\},}\{A=a_2\}\leftrightarrow \{B=b_2\},\{A=a_3\},\{A=a_3\}\rightarrow \{C=c_2\}\}$ which is inconsistent given that $\{C=c_2\}$ implies $\{D=d_1\}$ while we have $\{D=d_2\}$.  This branch is terminated and denoted as $\otimes$ in Figure~\ref{fig_search_lp_seatch_tree}.  Similarly, we find that the branches for $B=b_2$ and $B=b_3$ are also terminated. Call 1 ends at line 18, knowing that we still do not have a solution.

In call 0, let us continue the loop (lines 3-16) for $A=a_2$.  Strengthening $C_{orig}$ with $A=a_2$ gives $C_{a_2}=\{\{A=a_1\}\rightarrow \{C=c_2\},\{B=b_1\}\rightarrow \{D=d_2\},\{C=c_2\}\leftrightarrow \{D=d_1\},\{B=b_3\}\rightarrow \{D=d_2\},\{D=d_2\}\rightarrow \{A=a_1\},\{B=b_2\},\{A=a_3\}\leftrightarrow \{B=b_1\},\{A=a_3\}\rightarrow \{C=c_2\}\}$.  Since $C_{a_2}$ is not inconsistent, line 6 is ignored.  In line 8, the partial assignment induced by $A=a_2$ and $C_{a_2}$ is $K_{a_2}=a_2 b_2$ which means that variables $A$ and $B$ are instantiated.  We find the Reduced LP-tree $L_{a_2 b_2}$ which is shown in Figure~\ref{fig_search_lp_compatible_lp_trees}~(b).  In line 14, Search-LP($L_{a_2 b_2},C_{a_2},a_2 b_2$) is called (call 2).

For call 2, $D$ is the root node and its domain preference is $d_2\succ d_1$.  Consider the loop (lines 3-16) for $D=d_2$.  After strengthening $C_{a_2}$ with $D=d_2$ in line 4, we get $C_{a_2 d_2}=\{\{A=a_1\}\rightarrow \{C=c_2\},\{C=c_2\}\leftrightarrow \{D=d_1\},\{A=a_1\},\{B=b_2\},\{A=a_3\}\leftrightarrow \{B=b_1\},\{A=a_3\}\rightarrow \{C=c_2\}\}$ which is inconsistent for the value of $A$. Now consider the loop (lines 3-14) for $D=d_1$.  After strengthening $C_{a_2}$ with $D=d_1$ in line 4, we get $C_{a_2 d_1}=\{\{A=a_1\}\rightarrow \{C=c_2\},\{B=b_1\}\rightarrow \{D=d_2\},\{C=c_2\},\{B=b_3\}\rightarrow \{D=d_2\},\{D=d_2\}\rightarrow \{A=a_1\},\{B=b_2\},\{A=a_3\}\leftrightarrow \{B=b_1\},\{A=a_3\}\rightarrow \{C=c_2\}\}$ which is not inconsistent.  In line 8 the partial assignment induced by $D=d_1$ and $C_{a_2 d_1}$ is $K_{a_2 d_1}=d_1 c_2$.  At this point the Reduced LP-tree of $L_{a_2 b_2}$, given the instantiation of $D$ and $C$, become empty.  This indicates that we obtain the solution.  The solution $a_2 b_2 d_1 c_2$ is stored in $S$ in line 11.  The rest of the search tree is irrelevant which is denoted by $T$ in Figure~\ref{fig_search_lp_seatch_tree}.  Search-LP is terminated in line 12.

\subsection{Search-LP: Formal properties}
We establish the correctness of Search-LP using the theorem below.
\begin{theorem} Let $L$ be an LP-tree and $C$ be a set of constraints over $V$.  Search-LP returns the outcome $o$ if and only if $o$ is the most preferable feasible outcome with respect to $L$ and $C$.
\end{theorem}
\begin{proof} Using Lemma~\ref{lemma_compatible_lp_tree_correctness}, a Reduced LP-tree is also a Compatible LP-tree which indicates that the preferences induced by the Reduced LP-tree in line 9 are also held by the original LP-tree, given the instantiation.  Therefore, in order to prove this theorem, we have to prove that: (1) for every $o_1\succ o$, $o_1$ is infeasible with respect to $C$; and (2) $o$ is preferred to every other feasible outcome $o_2$.  We prove these below.

(1)  This proof is by contradiction.  Assume that $o_1$ is feasible with respect to $C$.  Let $X$ be the variable such that $An(X)$ gives the same value $R$ to $o_1$ and $o$, and $X$ gives $x_1$ and $x$ to $o_1$ and $o$ correspondingly.  Therefore, we get that $o=Rxo{'}$ and $o_1=Rxo_{1}{'}$.  Since $o_1\succ o$, we get $x_1\succ x$ (by Definition~\ref{definition_lex_order}). According to line 3 of Search-LP, the branch $X=x_1$ is created before the branch $X=x$ is created, in the search tree.  It is easy to see that the outcome $o_1$ is associated with the branch $X=x_1$.  Since $o_1$ is feasible, Search-LP returns $o_1$ and is terminated.  Search-LP does not return $o$ which is a contradiction.  Therefore, $o_1$ cannot be feasible.

(2)  In this case, Search-LP returns $o$, and then stops searching other outcomes.  $o_2$ is one of the feasible outcomes in the unsearched outcomes. Let $o$ and $o_2$ differ on the values of variable $X$ such that $An(X)$ gives the same value $R$ to both $o$ and $o_2$.  Let $X$ gives $x$ and $x_2$ to $o$ and $o_2$ correspondingly.  Since $o$ is searched by Search-LP before searching for $o_2$, we get that the branch $X=x$ is created before the branch $X=x_2$ is created, in the search tree.  This gives us $x\succ x_2$ (using line 3), which also implies $o\succ o_2$.
\end{proof}

The efficiency of standard solving algoritms for a CSP can often be improved using various variable and value order heuristics~\cite{Mouhoub2011VariableValueOrderingCSP,Yong2017VariableValueOrderCSP}. Such variable order heuristics are applicable in solving Constrained CP-nets, and thus the efficiency can be improved~\cite{Alanazi2016constrained_cp_net}. However, we cannot apply such heuristics in our Search-LP algorithm because the instantiation occurs according to a defined variable and value order by the corresponding LP-tree.

In theory, both the Constrained LP-tree and the CSP are NP-complete and require exponential time effort in the worst case to obtain the solution(s) using backtrack solving algorithms. On the other hand, in practice, the Constrained LP-tree is no harder than the CSP if the goal when solving the CSP is to find all the feasible solutions while Search-LP stops as soon as the first feasible outcome is found. However, if solving the CSP means finding a single feasible solution (which is usually the case), then solving the CSP is easier than finding the optimal outcome for the Constrained LP-tree as the CSP solver will also stop after finding the first feasible solution.

\section{A Divide and Conquer Algorithm for Dominance Testing in Acyclic CP-nets}
\label{Acyclic_CP_DT_Section_Algorithm_Example}

In this section, first, we describe a divide and conquer algorithm for dominance testing in acyclic CP-nets. Then, we illustrate the algorithm using some examples. Finally, we explain the formal properties of the algorithm, and compare the algorithm with the existing methods.

\subsection{Algorithm: Acyclic-CP-DT}
%

\ignore{We propose algorithm Acyclic-CP-DT to perform the dominance testing for an acyclic CP-net.} The algorithm Acyclic-CP-DT has three inputs.  The first input is a sub CP-net $N$ over the set of variables $V$, which is initially the original CP-net $N_{orig}$.  The second and third inputs, $o_1$ and $o_2$, are two partial assignments over $V$, which are initially two outcomes of $N_{orig}$.  Acyclic-CP-DT returns ``yes" if $N\models o_1\succ o_2$ holds, or ``no'', if $N\models o_1\succ o_2$ does not hold.

\begin{algorithm}
\textbf{Algorithm 3.} Acyclic-CP-DT($N,o_1,o_2$)\\
\textbf{Input:} An acyclic CP-net $N$ over the set of variables $V$; two outcomes $o_1$ and $o_2$ on $V$

\textbf{Output:} The algorithm returns ``yes'' if $N\models o_1\succ o_2$ holds, or ``no'' if $N\models o_1\succ o_2$ does not hold

\begin{algorithmic}[1]
  \STATE Let $X$ be a variable such that the ancestor set $An(X)$ in $N$ gives the same value $R$ to both $o_1$ and $o_2$, and $X$ gives two distinct values $x_1$ and $x_2$ to $o_1$ and $o_2$ correspondingly
  \STATE Let $V-(An(X)\cup\{X\})$ gives $o_1'$ and $o_2'$ to $o_1$ and $o_2$ correspondingly
  \IF {$x_2\succ x_1$ given $An(X)=R$ in $N$}
      \STATE \textbf{Return} ``no''
  \ELSIF {$x_1\succ x_2$ given $An(X)=R$ in $N$ and $o_1'=o_2'$}
      \STATE \textbf{Return} ``yes''
  \ELSE
      \STATE Construct a sub CP-net $N_1$ by removing $An(X)\cup \{X\}$ from $N$, and restricting the CPTs of the remaining variables to $Rx_1$
      \IF {Acyclic-CP-DT($N_1,o_1',o_2'$) = ``yes''}
          \STATE \textbf{Return} ``yes''
      \ENDIF
      \STATE Construct the sub CP-net $N_2$ by removing $An(X)\cup \{X\}$ from $N$, and restricting the CPTs of the remaining variables to $Rx_2$
      \IF {Acyclic-CP-DT ($N_2,o_1',o_2'$) = ``yes''}
          \STATE \textbf{Return} ``yes''
      \ENDIF
      \FOR {each $x_i\in D(X)-\{x_1,x_2\}$ such that $x_1\succ x_i \succ x_2$ given $An(X)=R$}
 	      \STATE Construct the sub CP-net $N_i$ by removing $An(X)\cup \{X\}$ from $N$, and restricting the CPTs of the remaining variables to $Rx_i$
          \IF {Acyclic-CP-DT($N_i,o_1',o_2'$) = ``yes''}
               \STATE \textbf{Return} ``yes''
          \ENDIF
      \ENDFOR
      \FOR {each $o_3'\in (D(V-(An(X)\cup\{X\}))-\{o_1',o_2'\})$ \ignore{such that a same value is assigned to $o_1'$, $o_2'$ and $o_3'$ by a bottom portion of an arbitrary topological order} }
          \IF {Acyclic-CP-DT($N_2,o_3',o_2'$) = ``yes''}
              \IF {Acyclic-CP-DT($N_1,o_1',o_3'$) = ``yes''}
                    \STATE \textbf{Return} ``yes''
               \ENDIF
          \ENDIF
      \ENDFOR
  \ENDIF
  \STATE \textbf{Return} ``no''
  \end{algorithmic}
\end{algorithm}

In line 1, a variable $X$ is selected such that $An(X)$ in $N$ gives the same value $R$ to both $o_1$ and $o_2$, and $X$ gives two distinct values $x_1$ and $x_2$ to $o_1$ and $o_2$ correspondingly. If $o_1\neq o_2$ is true, we claim that such a variable always exists for the acyclic CP-net $N$. We explain this in the following. Let $X_1X_2\cdots X_n$ be an arbitrary topological order on the variables of $N$. Every variable $X_i$ (where $i=1$ to $n$) gives $x_{1i}$ and $x_{2i}$ to $o_1$ and $o_2$ correspondingly. Then we get: $o_1=x_{11}x_{12}\cdots x_{1n}$ and $o_2=x_{21}x_{22}\cdots x_{2n}$. Since we have $o_1\neq o_2$, there exists $x_{1i}\neq x_{2i}$ such that $x_{11}x_{12}\cdots x_{1(i-1)} = x_{21}x_{22}\cdots x_{2(i-1)}$. In this case, $X_i$ gives $x_{1i}$ and $x_{2i}$ to $o_1$ and $o_2$ correspondingly, and $An(X_i)$ gives the same value to both $o_1$ and $o_2$.

Note that in the rest of the algorithm, to answer $N\models o_1\succ o_2$, no instantiation of $An(X)$ except $R$ is considered. Thus, the problem of answering $o_1\succ o_2$ in $N$ is converted to the problem of answering $o_1\succ o_2$ in $N$ restricted to $An(X)=R$. We see that if $R$ is not \emph{empty}, the converted problem is computationally easier than the original problem. We call this notion -- the prefix elimination. If $R$ is $empty$, no conversion occurs. However, as Acyclic-CP-DT is a recursive algorithm, the conversion and the amount of reduced computation will vary for the subsequent calls, depending on the input instances. In line 2, we denote the remaining value of $o_1$ and $o_2$ as $o_1'$ and $o_2'$ correspondingly.

The notion of ordering query~\cite{Boutilier2004_CPnetJAIR} is adapted to formulate the base condition in line 3. If the condition is true, Acyclic-CP-DT returns ``no'', i.e., $N\models o_1\succ o_2$ does not hold.  This indicates that in such cases, Acyclic-CP-DT can answer the dominance query using a single check. Given that the condition in line 3 is sufficient, but not necessary, to answer the ordering query $N\not\models o_1\succ o_2$, there can be cases in which $N\not\models o_1\succ o_2$ is true but this is not captured by the condition. To fill the gap and regard the completeness of the method, Acyclic-CP-DT returns ``no'' in line 30 if it does not answer ``yes'' to $N\models o_1\succ o_2$ in lines 6, 10, 14, 19 or 25 after evaluating the corresponding criterion of the lines.

The second base condition in line 5 is for a ``yes'' return.  We see that if the lower portion of the topological order gives the same value to $o_1$ and $o_2$ (i.e., $o_1'=o_2'$), the algorithm returns ``yes'', i.e., $N\models o_1\succ o_2$ holds. This implication is a direct consequence of the CP-net semantics, given that $o_1$ and $o_2$ differ only on the values of variable $X$. This is also related to the suffix fixing principle~\cite{Boutilier2004_CPnetJAIR} that the common value of two outcomes which is given by some bottom portion of a topological order can be ignored for determining the dominance query without affecting the completeness of the search. If none of the base cases is true, Acyclic-CP-DT continues (lines 8-29).  For the rest of the algorithm, we have $x_1\succ x_2$ given $An(X)=R$, and $o_1'\neq o_2'$.

In line 8, a sub CP-net $N_1$ by removing $An(X)\cup \{X\}$ from $N$, and restricting the CPTs of the remaining variables to $Rx_1$, is constructed.  Note that in the worst case, $N_1$ consists of at least a variable less than the variables in $N$.  Acyclic-CP-DT is called for this sub CP-net, considering the partial assignments $o_1'$ and $o_2'$ as the new set of inputs.  If this call returns ``yes'', the main call also returns ``yes''.  Otherwise, it does the similar thing for $Rx_2$ (lines 12-15).  Note that the two sub calls (in lines 9 and 13) can be written in a compact form, i.e., \textbf{if} Acyclic-CP-DT($N_1,o_1',o_2'$) = ``yes'' \textbf{or} Acyclic-CP-DT($N_2,o_1',o_2'$) = ``yes'' \textbf{then Return} ``yes".  However, we stress to write this Boolean expression in a separate form.  In the compact form, both Acyclic-CP-DT($N_1,o_1',o_2'$) and Acyclic-CP-DT($N_2,o_1',o_2'$) are called; and then the return values are evaluated to find the result. In the separate form, Acyclic-CP-DT($N_1,o_1',o_2'$) is called first.  If the return value is ``yes'', the solution is reached; and call to Acyclic-CP-DT($N_2,o_1',o_2'$) is not needed.  This way, the required computation is saved.  However, we do not know which call between these two should be executed first for a better performance of Acyclic-CP-DT, and so we choose one arbitrarily.

Returning ``yes'' in line 10 or line 14 indicates that there is an improving flipping sequence from $o_2$ to $o_1$ where every outcome consists of $x_1$ or $x_2$ as the value of $X$.  This idea is restrictive, and will potentially impact the completeness of Acyclic-CP-DT. Rather, we also need to search for a possible improving flipping sequence from $o_2$ to $o_1$ in which at least one outcome $o_i$ contains a value of $X$ except $x_1$ and $x_2$.  Let $x_i$ be such a value.  To complete the search, we construct the sub CP-net $N_i$ by removing $An(X)\cup \{X\}$ from $N$, and restricting the CPTs of the remaining variables to $Rx_i$, in line 17.  If $N_i$ induces $o_1'\succ o_2'$, we can show that $N$ also induces $o_1\succ o_2$, and ``yes'' is returned in line 19. Look at that for a Binary valued CP-net, we do not need to consider this part, i.e., lines 16-21 can be ignored. In lines 8, 12 and 16, every $x_j\in D(X)-\{x_1,x_2\}$ such that $x_j$ does not participate to prove $x_1\succ x_2$ using the transitive closure is considered to be irrelevant to query $N\models o_1\succ o_2$. This notion is related to the forward pruning technique. However, Boutilier et al.~\cite{Boutilier2004_CPnetJAIR} suggested the forward pruning as a pre-processing step to reduce the search space for testing a dominance query, while Acyclic-CP-DT considers removing the irrelevant values in every recursive call for the corresponding sub CP-net.

Lines 22-28 correspond to the search due to the distinct values of the outcomes given by the lower portion of a topological order.  This part of the algorithm is expensive. We attempt to find an outcome $o_3$ such that $o_2$ can be improved to $o_3$, and $o_3$ can be improved to $o_1$. Finding such an outcome indicates that $o_1\succ o_2$ holds, and Acyclic-CP-DT returns ``yes''.  Note that we still do not need to search the portion of the search space where $An(X)$ gives a different partial assignment from $R$.  It indicates that in this part, the algorithm is at least better than the methods, described in~\cite{Boutilier2004_CPnetJAIR}. Finally, if a ``yes'' answer to $N\models o_1\succ o_2$ is not found, Acyclic-CP-DT returns ``no'' in line 30.

\subsection{Example}
For simplicity, we consider the CP-net $N$ in Figure~\ref{fig_divide_conquer_cpnet_sub_cpnet}~(i).  Its topological order $A,B,C$ is unique.  Consider a call to Acyclic-CP-DT($N,a_2 b_2 c_2,a_1 b_1 c_1$) to check if $N\models a_2 b_2 c_2\succ a_1 b_1 c_1$ holds or not.  Variable $A$ gives two distinct values to the outcomes, and $CPT(A)$ represents $a_1\succ a_2$.  Therefore, in line 4, Acyclic-CP-DT($N,a_2 b_2 c_2,a_1 b_1 c_1$) returns ``no'', i.e., $N\models a_2 b_2 c_2\succ a_1 b_1 c_1$ does not hold.  Similarly, there can be many queries in the network which can be answered with a single check.

\begin{figure}
\centering
\includegraphics[width=2.6in]{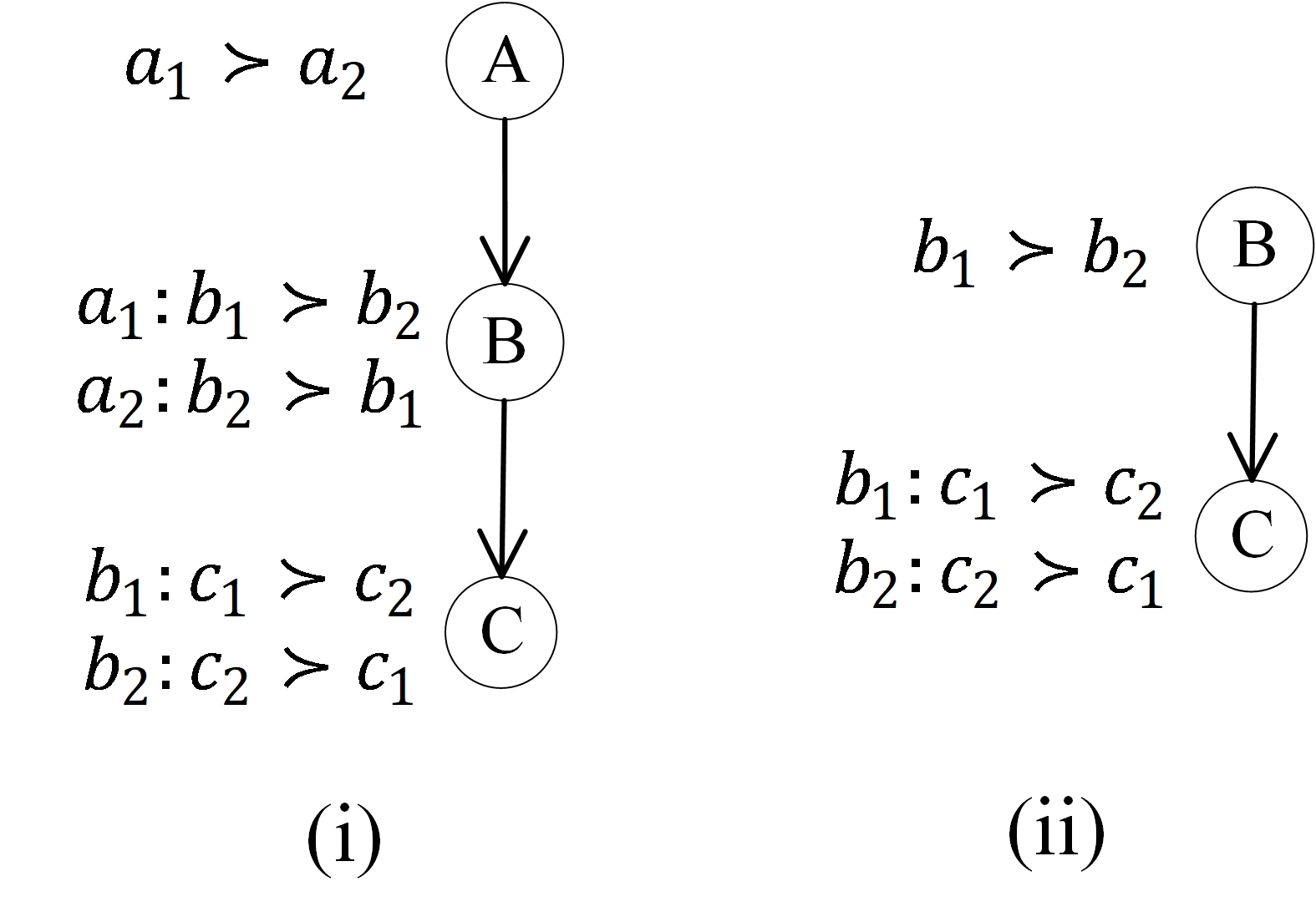}
\caption{(i) A CP-net and (ii) a sub CP-net.}
\label{fig_divide_conquer_cpnet_sub_cpnet}
\end{figure}

Consider the call to Acyclic-CP-DT($N,a_2 b_2 c_2,a_2 b_1 c_2$).  In this case, variable $B$ gives two distinct values to the outcomes while $A$ gives the same value $a_2$.  Given $A=a_2$, we have $b_2\succ b_1$.  Variable $C$ also gives the same value $c_2$ to both outcomes.  Therefore, in line 6, Acyclic-CP-DT($N,a_2 b_2 c_2,a_2 b_1 c_2$) returns ``yes''. In this way, Acyclic-CP-DT can simply answer any query in which two outcomes differ only on the values of a single variable.

Consider the call to Acyclic-CP-DT($N,a_1 b_1 c_1,a_2 b_1 c_2$).  For this call, the base conditions are not true.  Variable $A$ gives two distinct values to the outcomes.  We construct a sub CP-net $N_{a_1}$ by removing $A$ from $N$, and restricting $CPT(B)$ to $A=a_1$.  This sub CP-net is shown in Figure~\ref{fig_divide_conquer_cpnet_sub_cpnet}~(ii).  In line 9, Acyclic-CP-DT($N_{a_1},b_1 c_1,b_1 c_2$) is called.  For this call, $C$ gives two distinct values to the partial assignments $b_1 c_1$ and $b_1 c_2$.  Base condition in line 3 is not true as $c_2\succ c_1$ is not true for $B=b_1$; however in line 5, the condition is true.  Therefore, Acyclic-CP-DT($N_{a_1},b_1 c_1,b_1 c_2$) returns ``yes'', and consequently Acyclic-CP-DT($N,a_1 b_1 c_1,a_2 b_1 c_2$) also returns ``yes'', i.e., $N\models a_1 b_1 c_1\succ a_2 b_1 c_2$ holds.  Note that in the network, there are six improving flipping sequences from $a_2 b_1 c_2$ to $a_1 b_1 c_1$.  They are:\\

$1.$ $a_2 b_1 c_2\rightarrow a_1 b_1 c_2\rightarrow a_1 b_1 c_1$

$2.$ $a_2 b_1 c_2\rightarrow a_2 b_1 c_1\rightarrow a_1 b_1 c_1$

$3.$ $a_2 b_1 c_2\rightarrow a_2 b_2 c_2\rightarrow a_1 b_2 c_2  \rightarrow a_1 b_1 c_2\rightarrow a_1 b_1 c_1$

$4.$ $a_2 b_1 c_2\rightarrow a_2 b_1 c_1\rightarrow a_2 b_2 c_1  \rightarrow a_1 b_2 c_1\rightarrow a_1 b_1 c_1$

\hangindent=1cm \hangafter=1  $5.$ $a_2 b_1 c_2\rightarrow a_2 b_1 c_1\rightarrow a_2 b_2 c_1  \rightarrow a_2 b_2 c_2\rightarrow a_1 b_2 c_2\rightarrow a_1 b_1 c_2\rightarrow a_1 b_1 c_1$

\hangindent=1cm \hangafter=1 $6.$ $a_2 b_1 c_2\rightarrow a_2 b_1 c_1\rightarrow a_2 b_2 c_1  \rightarrow a_1 b_2 c_1\rightarrow a_1 b_2 c_2\rightarrow a_1 b_1 c_2\rightarrow a_1 b_1 c_1$ \\

However, Acyclic-CP-DT chooses the one with the least number of improving flips; in this example, it is $a_2 b_1 c_2\rightarrow a_1 b_1 c_2\rightarrow a_1 b_1 c_1$. This claim is understandable with the fact that Acyclic-CP-DT checks the easy ways first in lines 3, 5, 9 and 13. If Acyclic-CP-DT fails to answer the dominance query in the lines above, it checks the hard part in lines 16-21 and then in lines 22-28.

For instance, consider the call to Acyclic-CP-DT($N,a_1b_1c_2,a_2b_1c_1$).  We see that both Acyclic-CP-DT($N_{a_1},b_1c_2,b_1c_1$) and Acyclic-CP-DT($N_{a_2},b_1c_2,b_1c_1$) return ``no'', where $N_{a_1}$ and $N_{a_2}$ are obtained from $N$ by removing $A$ and restricting $CPT(B)$ to $A=a_1$ and $A=a_2$ correspondingly.  Therefore, Acyclic-CP-DT cannot obtain a solution using the conditions in lines 3, 5, 9 or 13.  Lines 16-21 are irrelevant as $N$ is a binary valued CP-net. We evaluate the lines 22-28.  Except $b_1c_2$ and $b_1c_1$, $\{B,C\}$ has two partial assignments which are $b_2c_1$ and $b_2c_2$. We can show that both Acyclic-CP-DT($N_{a_1},b_1c_2,b_2c_1$) and Acyclic-CP-DT($N_{a_2},b_2c_1,b_1c_1$) return ``yes''.  Therefore, the main call also returns ``yes'' in line 25, which indicates $N\models a_1b_1c_2\succ a_2b_1c_1$ holds. In this case, the improving flipping sequence is $a_2b_1c_1 \rightarrow a_2b_2c_1 \rightarrow a_1b_2c_1 \rightarrow a_1b_2c_2 \rightarrow a_1b_1c_2$.

\subsection{Acyclic-CP-DT: Formal Properties}

\label{Acyclic_CP_DT_Section_Formal_properties}

Before we prove the correctness and the completeness of Acyclic-CP-DT, we provide the lemmas below.

\begin{lemma} Let $o_1$ and $o_2$ be two outcomes of an acyclic CP-net $N$ over the variables set $V$.  Let $X$ be a variable such that $An(X)$ in $N$ gives the same value $R$ to both $o_1$ and $o_2$, and $X$ gives two different values $x_1$ and $x_2$ to $o_1$ and $o_2$ correspondingly.  Let $V-(An(X)\cup \{X\})$ gives $o_1'$ and $o_2'$ to $o_1$ and $o_2$ correspondingly.  Let $N_1$ and $N_2$ be two sub CP-nets obtained from  $N$ by removing $An(X)\cup \{X\}$ and restricting the CPTs of the remaining variables to $Rx_1$ and $Rx_2$ correspondingly.  Then, the following are true if we have $x_1\succ x_2$ for $An(X)=R$:

\begin{enumerate}
	\item $(N_1\models o_1'\succ o_2')\Rightarrow (N\models o_1\succ o_2 )$.
	\item $(N_2\models o_1'\succ o_2')\Rightarrow (N\models o_1\succ o_2 )$.
\end{enumerate}
\label{lemma_divide_conquer_yes_return1}
\end{lemma}

\begin{proof} $1.$ From CP-net semantics, we get:  $Rx_1 o_2'\succ Rx_2 o_2'$.  Therefore, we have:\\

$(N_1\models o_1'\succ o_2')\land (Rx_1 o_2'\succ Rx_2 o_2') $\\
$\Rightarrow (Rx_1 o_1'\succ Rx_1 o_2')\land (Rx_1 o_2'\succ Rx_2 o_2' ) $\\
$\Rightarrow Rx_1 o_1'\succ Rx_2 o_2'$ [Transitivity]\\
$\Rightarrow o_1\succ o_2 $.\\

\noindent $2.$ From CP-net semantics, we get: $Rx_1 o_1'\succ Rx_2 o_1'$.  Therefore, we have:\\

$(N_2\models o_1'\succ o_2')\land (Rx_1 o_1'\succ Rx_2 o_1')$\\
$\Rightarrow (Rx_2 o_1'\succ Rx_2 o_2')\land (Rx_1 o_1'\succ Rx_2 o_1' ) $\\
$\Rightarrow Rx_1 o_1'\succ Rx_2 o_2'$ [Transitivity]\\
$\Rightarrow o_1\succ o_2$. 
\end{proof}

To answer the ordering queries in acyclic CP-nets, Boutilier et al.~\cite{Boutilier2004_CPnetJAIR} established a sufficient but not necessary condition. We extend the condition in case each CPT in the CP-net does not necessarily represent a total order over the corresponding variable's values for each combination of the parent set. The preference order on the values can also be a strict partial order~\cite{Ahmed2018PartialCPnet}. For example, we consider that if CPT($X$) for variable $X$ (where $x_1,x_2\in D(X)$) does not represent $x_1\succ x_2$, then $x_2\succ x_1$ is true, or $x_1$ and $x_2$ are incomparable. The latter can be interpreted as the user did not provide a preference between $x_1$ and $x_2$. The lemma below is an extension of Corollary 4 in~\cite{Boutilier2004_CPnetJAIR}.

\begin{lemma} Let $o_1$ and $o_2$ be two outcomes of an acyclic CP-net $N$ over $V$.  Let $X$ be a variable such that $An(X)$ in $N$ gives the same value $R$ to both $o_1$ and $o_2$, and $X$ gives two different values $x_1$ and $x_2$ to $o_1$ and $o_2$ correspondingly. Let $V-(An(X)\cup \{X\})$ gives $o_1'$ and $o_2'$ to $o_1$ and $o_2$ correspondingly. Then, if $CPT(X)$ does not represent $x_1\succ x_2$ for $An(X)=R$, then we get that $N\not\models o_1\succ o_2$ is true.
\label{lemma_divide_conquer_no_return1}
\end{lemma}

\begin{proof} We prove this by contradiction.  Assume that $N\models o_1\succ o_2$ holds.  Therefore, we get $Rx_1 o_1'\succ Rx_2 o_2'$. This means that there exists $Rx_2 o_1'$ such that $Rx_2 o_1'$ can be improved to $Rx_1 o_1'$.  This contradicts with $CPT(X)$ that does not represent $x_1\succ x_2$ for $An(X)=R$.  Therefore, $N\models o_1\succ o_2$ cannot hold. 
\end{proof}

\begin{lemma} Let $o_1$ and $o_2$ be two outcomes of an acyclic CP-net $N$ over $V$.  Let $X$ be a variable such that $An(X)$ in $N$ gives the same value $R$ to both $o_1$ and $o_2$, and $X$ gives two different values $x_1$ and $x_2$ to $o_1$ and $o_2$ correspondingly. Let $V-(An(X)\cup \{X\})$ gives $o_1'$ and $o_2'$ to $o_1$ and $o_2$ correspondingly.  Let $o_1'\succ o_2'$ is not true for $An(X)\cup \{X\}=Rx_2$ or $An(X)\cup \{X\}=Rx_1$.  If we have $x_1\succ x_2$, $o_1'\neq o_2'$ and there is no $o_3'\in (D(V-(An(X)\cup \{X\}))-\{o_1',o_2'\})$ such that  $o_3'\succ o_2'$ for $An(X)\cup \{X\}=Rx_2$ and $o_1'\succ o_3'$ for $An(X)\cup \{X\}=Rx_1$, then $N\models o_1\succ o_2$ does not hold.
\label{lemma_divide_conquer_no_return_2}
\end{lemma}

\begin{proof} Assume that $N\models o_1\succ o_2$ holds.  Therefore, there is an improving flipping sequence from $Rx_2 o_2'$ to $Rx_1 o_1'$.  This indicates that $Rx_2 o_2'$ can be improved to some $Rx_2 o_3'$, and $Rx_2 o_3'$ can be improved to $Rx_1 o_1'$.  In other words, we have $o_3'\succ o_2'$ for $An(X)\cup \{X\}=Rx_2$ and $o_1'\succ o_3'$ for $An(X)\cup \{X\}=Rx_1$.  However, this contradicts with our assumption.  Therefore, $N\models o_1\succ o_2$ cannot hold. 
\end{proof}

We now provide our main result -- correctness and completeness of Acyclic-CP-DT.

\begin{theorem}~(Correctness) Let $o_1$ and $o_2$ be two outcomes of an acyclic CP-net $N$ over $V$.  Acyclic-CP-DT($N,o_1,o_2$) returns ``yes'' if $N\models o_1\succ o_2$ holds.
\end{theorem}

\begin{proof} In call Acyclic-CP-DT($N,o_1,o_2$), let $X$ be a variable such that $An(X)$ in $N$ gives the same value $R$ to both $o_1$ and $o_2$, and $X$ gives two different values $x_1$ and $x_2$ to $o_1$ and $o_2$ correspondingly.  Let $V-(An(X)\cup \{X\})$ gives $o_1'$ and $o_2'$ to $o_1$ and $o_2$ correspondingly.

This proof is by induction on the number of variables.  First, we prove two base cases in lines 4 and 6.  (1) If $CPT(X)$ represents $x_2\succ x_1$ for $An(X)=R$, $CPT(X)$ does not represent $x_1\succ x_2$ for $An(X)=R$. According to Lemma~\ref{lemma_divide_conquer_no_return1}, $N\models o_1\succ o_2$ does not hold and Acyclic-CP-DT($N,o_1,o_2$) correctly returns ``no'', in line 4.  (2) If we have $x_1\succ x_2$ for $An(X)=R$ and $o_1'=o_2'$, this indicates that the outcomes $Rx_1 o_1'$ and $Rx_2 o_2'$ differ only on the value of $X$.  By CP-net semantics, we get $Rx_1 o_1'\succ Rx_2 o_2'$ which is $o_1\succ o_2$.  Therefore, Acyclic-CP-DT($N,o_1,o_2$) correctly returns ``yes'' in line 6.

Let us consider the inductive hypothesis that Acyclic-CP-DT returns correct result for any acyclic CP-net with less than $n$ variables.  Let $N$ be the acyclic CP-net with $n$ variables. We show in the following that Acyclic-CP-DT correctly returns ``yes'' in lines 10, 14, 19 and 25.

According to the inductive hypothesis, call Acyclic-CP-DT($N_1,o_1',o_2'$) in line 9 is correct as $N_1$ is obtained from $N$ by removing at least a variable.  Similarly, call Acyclic-CP-DT($N_2,o_1',o_2'$) in line 13 is correct.  Now, we get that Acyclic-CP-DT($N_1,o_1',o_2'$) = ``yes'' implies $N_1\models o_1'\succ o_2'$ and Acyclic-CP-DT($N_2,o_1',o_2'$) = ``yes'' implies $N_2\models o_1'\succ o_2'$.  Using Lemma~\ref{lemma_divide_conquer_yes_return1}, in both cases, we get $o_1\succ o_2$.  Therefore, Acyclic-CP-DT($N,o_1,o_2$) correctly returns ``yes'' in line 10 and line 14.

The inductive hypothesis also implies that the call in line 18 is correct. Then, we get:\\

Acyclic-CP-DT($N_i,o_1',o_2'$) = ``yes''\\
$\Rightarrow N_i\models o_1'\succ o_2'$\\
$\Rightarrow Rx_i o_1'\succ Rx_i o_2'$

\noindent By CP-net semantics (given $x_1\succ x_i\succ x_2$ for $An(X)=R$), we get:  $Rx_1o_1'\succ Rx_io_1'$ and $Rx_io_2'\succ Rx_2o_2'$ hold.  Then, the principle of transitivity closure implies $Rx_1o_1'\succ Rx_2o_2'$ which is $o_1\succ o_2$.  Therefore, Acyclic-CP-DT correctly returns ``yes'' in line 19.

Now, consider that $o_3'$ is a partial assignment on $V-(An(X)\cup \{X\})$ except $o_1'$ and $o_2'$.  If both conditions in lines 23 and 24 are true, then we get:\\

 \hangindent=1.5cm  (Acyclic-CP-DT($N_2,o_3',o_2'$) = ``yes'') $\land$ (Acyclic-CP-DT($N_1,o_1',o_3'$) = ``yes'')

\noindent $\Rightarrow (N_2\models o_3'\succ o_2')\land (N_1\models o_1'\succ o_3')$

\noindent $\Rightarrow (N\models Rx_2 o_3'\succ Rx_2 o_2')\land (N\models Rx_1 o_1'\succ Rx_1 o_3')$

\noindent $\Rightarrow (N\models Rx_2 o_3'\succ Rx_2 o_2')\land (N\models Rx_1 o_1'\succ Rx_1 o_3') \land (N\models Rx_1 o_3'\succ Rx_2 o_3')$ [Using CP-net semantics]

\noindent $\Rightarrow N\models Rx_1 o_1'\succ Rx_2 o_2'$ [Transitivity]

\noindent $\Rightarrow N\models o_1\succ o_2$ \\

\hangindent=0cm \noindent Therefore, Acyclic-CP-DT($N,o_1,o_2$) correctly returns ``yes'' in line 25. 
\end{proof}

\begin{theorem}~(Completeness) Let $o_1$ and $o_2$ be two outcomes of an acyclic CP-net $N$ over $V$.  Acyclic-CP-DT($N,o_1,o_2$) returns ``no'' if $N\models o_1\succ o_2$ does not hold.
\label{theorem_divide_conquer_completeness}
\end{theorem}

\begin{proof}
Acyclic-CP-DT($N,o_1,o_2$) can return ``no'' in line 4 or line 30.  If ``no'' is returned in line 4, Lemma~\ref{lemma_divide_conquer_no_return1} indicates that $N\models o_1\succ o_2$ does not hold, and this return is correct.

Now, we consider line 30.  We see that Acyclic-CP-DT($N,o_1,o_2$) searches the entire search space except for three cases described below. If the search finishes and a solution is not obtained, Acyclic-CP-DT($N,o_1,o_2$) returns ``no'' in line 30. We show below that for each of the three cases, $N\not\models o_1\succ o_2$ does not hold.

The first case is that for every $R'\in D(An(X))-\{R\}$, the search is ignored.  We can easily see that it does not impact the completeness, as $Rx_2o_2'$ cannot be improved to $R'x_2o_2'$ such that $R'x_2o_2'$ can be improved to $Rx_1o_1'$. More specifically, by removing the variables $V-An(X)$ and their CPTs from the original CP-net $N$, we can get a CP-net on $An(X)$ in which, if $R\succ R'$ holds then $R'\succ R$ does not hold.

The second case is:  $x_1\succ x_2$, $o_1'\neq o_2'$, $N_1\not\models o_1'\succ o_2'$, $N_2\not\models o_1'\succ o_2'$ and there is no $x_i\in D(X)-\{x_1,x_2\}$ such that $CPT(X)$ represents $x_1\succ x_i \succ x_2$ for $An(X)=R$ (See lines 16-21).  In this case, using Lemma~\ref{lemma_divide_conquer_no_return1}, we get, $(N\not\models Rx_1o_1'\succ Rx_io_2')\lor (N\not\models Rx_io_1'\succ Rx_2o_2')$.  This indicates that there cannot be any improving flipping sequence from $Rx_2o_2'$ to $Rx_1o_1'$ such that the sequence consists of an outcome with value $x_i$.  Therefore, ignoring the search for $X=x_i$ does not impact the completeness of Acyclic-CP-DT.

The third case is:  $x_1\succ x_2$, $o_1'\neq o_2'$, $N_1\not\models o_1'\succ o_2'$, $N_2\not\models o_1'\succ o_2'$ and there is no $o_3'\in (D(V-(An(X)\cup\{X\}))-\{o_1',o_2'\})$ such that $o_3'\succ o_2'$ for $An(X)\cup \{X\}=Rx_2$ and $o_1'\succ o_3'$ for $An(X)\cup \{X\}=Rx_1$ (See lines 22-28).  For this case, using Lemma~\ref{lemma_divide_conquer_no_return_2}, we see that $N\models o_1\succ o_2$ does not hold.  Therefore, Acyclic-CP-DT($N,o_1,o_2$) correctly returns ``no'' in line 30. 
\end{proof}





Acyclic-CP-DT has the exponential time complexity, in the worst case, i.e., the time complexity is $O(m^{n^2})$ where $n$ is the number of variables, and $m$ is the number of values that each variable is bounded by. However, by introducing the concept of prefix elimination, Acyclic-CP-DT outperforms the existing methods~\cite{Boutilier2004_CPnetJAIR}. We discuss this in the following subsection.

\subsection{Comparison with the existing methods}

\label{Acyclic_CP_DT_Section_Comparison}

Boutilier et al.~\cite{Boutilier2004_CPnetJAIR} converted the problem of dominance testing to searching for improving flipping sequences. They proposed two pruning techniques, i.e., suffix fixing and forward pruning. These techniques can be applied to any generic search algorithm. In this subsection, we describe how Acyclic-CP-DT implicitly  involves the suffix fixing and the forward pruning principles. Then, we describe the notion of prefix elimination, introduced in Acyclic-CP-DT.  This property adds additional computational advantage to the pruning principles above.

We represent two outcomes $o_1$ and $o_2$ as $o_1=Rx_1o_1'$ and $o_2=Rx_2o_2'$, where (1) variable $X$ gives two distinct values $x_1$ and $x_2$ to $o_1$ and $o_2$ correspondingly, (2) the top portion of the topological order (more specifically the ancestor set of $X$) gives $R$ to both $o_1$ and $o_2$, and (3) the bottom portion gives $o_1'$ and $o_2'$ to $o_1$ and $o_2$ correspondingly. Recall the termination conditions in lines 3 and 5. Let us consider that $o_1'=o_2'$ is true. Then, both conditions return the answer based on the values $Rx_1$ and $Rx_2$. The values $o_1'$ and $o_2'$ do not participate in determining the dominance query, which is exactly the suffix fixing rule, described in~\cite{Boutilier2004_CPnetJAIR}.

On the other hand, if we have $x_1\succ x_2$ or $x_2\succ x_1$ for $An(X)=R$ and the termination conditions (in lines 3 and 5) are not true, then the sub CP-nets are formed by restricting the original CP-net to the relevant values (in this case, the relevant values are $x_1$ and $x_2$). More specifically, Acyclic-CP-DT ignores every irrelevant value $x\in D(X)-\{x_1,x_2\}$ to the dominance query. This method is the forward pruning that Boutilier et al.~\cite{Boutilier2004_CPnetJAIR} suggested to apply as a pre-processing step in the CP-net before searching for improving flipping sequences. However, we incorporate the forward pruning in every recursive call of Acyclic-CP-DT.  Note that a value $x_i\in D(X)-\{x_1,x_2\}$ is relevant if $x_1\succ x_i \succ x_2$ or $x_2\succ x_i \succ x_1$. In this case, we do not ignore $x_i$ (see lines 16-20).

Given a dominance query $N\models o_1\succ o_2$, in the existing methods~\cite{Boutilier2004_CPnetJAIR} for finding an improving flipping sequence from $o_2$ to $o_1$, a value $x\in D(X)$ in $o_2$ for every variable $X\in V$ is checked for a possible improvement to a better value of $X$.  This improvement process continues until a solution is obtained or a further improvement is not possible.  However, for an instance of Acyclic-CP-DT($N,o_1,o_2$) (where $X$ is a variable such that $An(X)$ in $N$ gives the same value $R$ to both $o_1$ and $o_2$, and $X$ gives two different values $x_1$ and $x_2$ to $o_1$ and $o_2$ correspondingly), we see that no value in $R$ is checked for a possible improvement to a better value. If some values in $R$ are improved to get $R'\in (D(An(X))-\{R\})$, then there is an improving flipping sequence from $R$ to $R'$ for the CP-net, obtained by removing all variables in $V-An(X)$. Then, we cannot have $R\succ R'$. Since both outcomes ($o_1$ and $o_2$) consist of $R$, trying the improving flips for the values in $R$ will never give an improving sequence from $o_2$ to $o_1$. In this way, Acyclic-CP-DT can reduce the search space without impacting its completeness as we see in Theorem~\ref{theorem_divide_conquer_completeness}.  We call this notion -- the prefix elimination.

Santhanam et al.~\cite{Santhanam2010dominanceModelChecking} showed that the dominance testing problem can be successfully represented using model checking before it is solved.  However, they did not consider any actual improvement of the methods for answering the dominance query.

\section{Conclusion and Future Work}
\label{SEC_CON_FUTURE_WORK}

We formally illustrate the variable importance order induced by the parent-child relation of an acyclic CP-net. We show that the acyclic CP-net represents a total order over the outcomes if and only if the induced variable order is total. Then, we propose an extension of the CP-net model that we call the CPR-net.  The CPR-net consists of an acyclic CP-net augmented with a set of Additional Relative Importance (ARI) statements, each for every pair of variables which are not ordered by the CP-net. Indeed, the CP-net model is more expressive than the CPR-net model.  However, an acyclic CPR-net guarantees a total order over the outcomes.  Therefore, the constrained optimization using the acyclic CPR-net involves a single optimal solution.  Consequently, the dominance testing between outcomes is not needed as shown in our proposed solving algorithm.

We have proposed the Search-LP algorithm to obtain the most preferable feasible outcome for a Constrained LP-tree.  Search-LP instantiates the variables and values according to the hierarchical orders of variables and values induced by the conditional relative importance relations and the conditional preferences which are defined in the LP-tree.  In comparison with the Constrained CP-nets or the Constrained TCP-nets, the Constrained LP-tree is superior in the fact that the latter does not require dominance testing that is a very expensive operation in both CP-nets and TCP-nets.

We have proposed the divide and conquer algorithm, that we call Acyclic-CP-DT to perform the dominance testing in acyclic CP-nets. First, Acyclic-CP-DT determines a variable such that its ancestor set gives the same value (or the partial assignment) to both outcomes. In searching for improving flipping sequences, this value of ancestor set is ignored (i.e., no variable in the ancestor set is tried for a possible improving flip), although this value is used to build the sub CP-nets. In fact, our method disregards every value of the ancestor set except the one above, and we call this concept the prefix elimination. The prefix elimination reduces the search space without impacting the completeness of Acyclic-CP-DT. Then, Acyclic-CP-DT is recursively called for the sub CP-nets until an improving flipping sequence is searched. Together with the prefix elimination, we show that Acyclic-CP-DT also involves the sufix fixing and the forward pruning principles~\cite{Boutilier2004_CPnetJAIR}, which justifies the computational advantage of Acyclic-CP-DT over the existing methods.

Acyclic-CP-DT extends to the Partial CP-nets~\cite{Boutilier2004_CPnetJAIR,Ahmed2018PartialCPnet} where a CPT represents a partial order, not necessarily a total order, over the variable values for each combination of the variable's parent set.  In this case, it is easy to see that Lemma~\ref{lemma_divide_conquer_no_return1} also holds when $x_1$ and $x_2$ are incomparable.  Therefore, to apply Acyclic-CP-DT for Partial CP-nets, we need to simply change the condition in line 3 to $x_2\succ x_1$, or $x_2$ and $x_1$ are incomparable.

There is one and only one preference order over the outcomes, which satisfies an acyclic CPR-net. However, the satisfiability of a cyclic CPR-net is not guaranteed.  For example, if we consider the ARI $C\rhd A$ with the acyclic CP-net of Figure~\ref{Label_totallydependentcpnet_fig}, we get a cyclic CPR-net.  The CP-net in Figure~\ref{Label_totallydependentcpnet_fig} induces $a_1 b_1 c_2\succ a_2 b_1 c_1$, while it can be shown that the ARI $C\rhd A$ induces $a_2 b_1 c_1\succ a_1 b_1 c_2$.  Therefore, both of the outcomes dominate each other, in the CPR-net, one using the CP-net component and another using the ARI component.  This indicates that there is no preference order that satisfies the CPR-net. Nevertheless, given a set of hard constraints, some of the outcomes might be infeasible.  In this case, we still may have the optimal outcome for the CPR-net and the set of constraints.  Unfortunately, Search-CPR cannot find the optimal outcome, as it will fail at the beginning of execution due to the directed graph not having a topological order.  In this regard, we plan to find a solving algorithm that can deal with cyclic CPR-nets.

Acyclic-CP-DT does not work for cyclic CP-nets as it requires an arbitrary topological order of the related DAG.  We do not know if a divide and conquer algorithm can be derived for dominance testing in cyclic CP-nets.

We plan to incorporate various constraint propagation techniques~\cite{Dechter2003constraint,Zhang2016ConstrainedTcpNet,mouhoub2004systematic,mouhoub2008managing} with Search-CPR and Search-LP as was done in~\cite{Alanazi2016constrained_cp_net} for Constrained CP-nets. We will then run a series of experiments on Constrained CPR-net and Constrained LP-tree instances in order to assess the time efficiency of Search-CPR and Search-LP with constraint propagation. In this regard, we plan to extend the Model RB~\cite{Xu2000_RB_model}, which has the ability to generate hard constraints to solve CSP instances, to Constrained CPR-net and Constrained LP-tree instances.  This will enable us to measure the performance of our Search-CPR and Search-LP algorithms on hard to solve instances including those near the phase transition.

\section*{acknowledgements}
This research was partially funded by Natural Sciences and Engineering Research Council of Canada (RGPIN-2016-05673).

\bibliographystyle{plain}
\bibliography{CPRLP}

\clearpage




\end{document}